\documentclass[runningheads]{llncs}

\usepackage{eccv}

\usepackage{eccvabbrv}

\usepackage{booktabs}
\usepackage{caption}
\usepackage{graphicx}
\usepackage{booktabs}
\usepackage{amsmath}
\usepackage{amssymb}
\usepackage{mathtools}
\usepackage{multirow}
\usepackage{makecell}
\usepackage{bm}
\usepackage{tikz}
\usepackage[ruled, vlined]{algorithm2e}

\newcommand{\para}[1]{\vspace{.05in}
\noindent
\textbf{#1}\quad}

\usepackage[accsupp]{
    axessibility
} 

\usepackage[pagebackref]{hyperref}

\usepackage{orcidlink}

\begin{document}
    \title{Compress Then Adapt? No, Do It Together via Task-aware Union of Subspaces}
    \titlerunning{Joint Compression and Adaptation}

    \author{
    Jingze Ge\inst{1} \and
    Yun Liu\inst{2} \and
    Xue Geng\inst{3} \and
    Wanqi Dong\inst{1} \and
    Wang Zhe Mark\inst{3} \and
    Min Wu\inst{3} \and
    Xulei Yang\inst{3}
    }
    \authorrunning{Ge et al.}
    
    \institute{
    National University of Singapore \and
    Nankai University \and
    Institute for Infocomm Research (I$^2$R), A*STAR \\
    \email{
    jingze.ge@u.nus.edu,
    liuyun@mail.nankai.edu.cn,
    geng\_xue@i2r.a-star.edu.sg,
    wanqi.dong@u.nus.edu,
    wumin@a-star.edu.sg,
    YANG\_XULEI@I2R.A-STAR.EDU.SG
    }
    }
    \maketitle

    \begin{abstract}
    Adapting large pretrained models to diverse tasks is now routine, yet the two dominant strategies of parameter-efficient fine-tuning (PEFT) and low-rank compression are typically composed in sequence. This decoupled practice first compresses and then fine-tunes adapters, potentially misaligning the
    compressed subspace with downstream objectives and squandering a global
    parameter budget. To overcome this limitation, we introduce \textbf{JACTUS} (Joint
    Adaptation and Compression with a Task-aware Union of Subspaces), a single
    framework that unifies compression and adaptation. From a small calibration set,
    JACTUS estimates input and pre-activation gradient covariances, forms their orthogonal
    union with the pretrained weight subspace, performs a projected low-rank
    approximation inside this union, allocates rank globally by marginal gain per
    parameter, and trains only a compact core matrix. This explicitly mitigates the potential misalignment between the compressed subspace and downstream objectives by coupling the directions preserved for compression with those required for adaptation, yielding a deployable low-rank model that avoids retaining full frozen weights while enabling fast and robust tuning. On vision, JACTUS attains an average 89.2\% accuracy on ViT-Base across eight datasets at 80\% retained parameters,
    surpassing strong 100\% PEFT baselines (\eg, DoRA 87.9\%). On language, JACTUS
    achieves an 80.9\% average on Llama2-7B commonsense QA at the same 80\% retained-parameter budget,
    outperforming 100\% PEFT (\eg, DoRA 79.7\%) and exceeding prior compress-then-finetune
    pipelines under the same ratained-parameter budget. We will release code.
\end{abstract}
    \section{Introduction}
Modern vision and language models~\cite{dosovitskiy2021image, liu2021Swin, DBLP:journals/corr/abs-2302-13971,
touvron2023llama2} deliver strong performance but are expensive to adapt and deploy
at scale~\cite{wan2024efficientlargelanguagemodels,zhou2024efficientinference}.
Parameter-Efficient Fine-Tuning (PEFT) methods such as LoRA \cite{hu2022lora}
and its variants~\cite{liu2024dora,wang2023orthogonalsubspacelearninglanguage,zhang2023adaptive,meng2024pissa}
greatly reduce training cost by introducing low-rank adapters on top of frozen
weights, yet at inference time the full pre-trained weights must still be retained. Conversely, low-rank compression via Singular Value Decomposition (SVD)-family techniques~\cite{hsu2022languagemodelcompressionweighted,yuan2023asvd,wang2025svdllm,qinsi2025dobisvd}
reduces the number of deployed parameters, but typically selects directions by weight-space
energy or activation heuristics~\cite{amari1998natural,martens2015kfac},which only indirectly reflect downstream objectives. In practice, these two lines of work are often composed in sequence, where one first compresses a pretrained model and then fine-tunes it~\cite{wang2025svdllm,Hwang2024PCLoRA,Zhao2024CALoRA,zhang-etal-2024-loraprune}.
However, this seemingly modular pipeline exposes a fundamental tension between subspace preservation and task adaptation.
Compression typically selects a low-rank subspace to reduce reconstruction error or to preserve pretrained behavior, whereas downstream adaptation requires degrees of freedom that align with task-specific gradients.
When the compressed subspace is fixed in advance, the model is committed to a task-agnostic parameterization that can discard directions important for the downstream objective.
As a result, subsequent low-rank adaptation, even when it can be merged back into the compressed weights at deployment, may be intrinsically limited in recovering the lost task-relevant capacity, leading to persistent accuracy and robustness gaps under a fixed global budget.
This limitation becomes more pronounced in resource-constrained settings where the budget must be allocated across layers in a principled manner.
Therefore, the central challenge is to select a shared low-rank subspace and allocate a global budget so that compression preserves task-relevant capacity while adaptation uses it effectively.

\begin{figure}[!t]
    \centering
    \includegraphics[width=\linewidth]{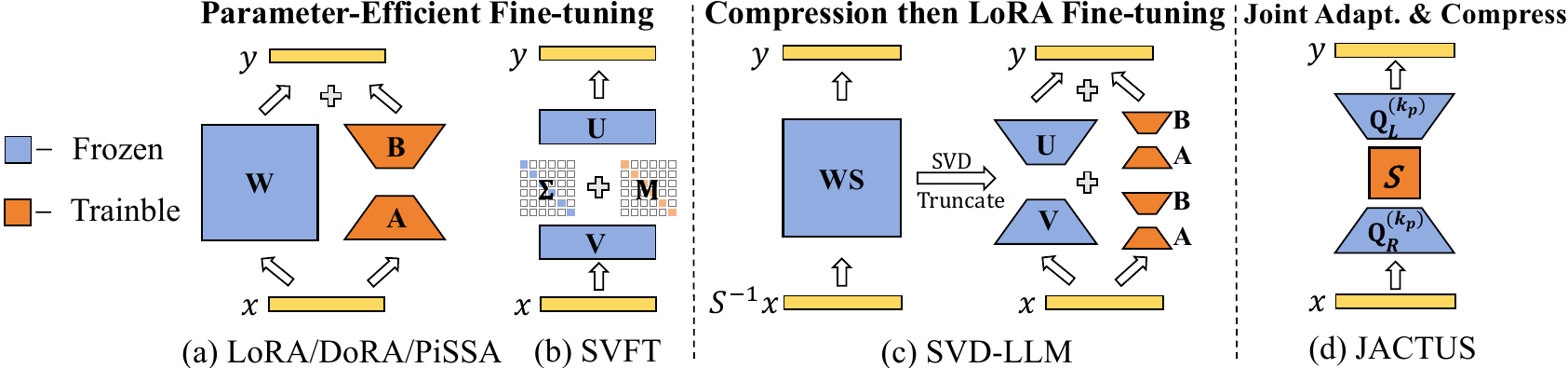}
    \caption{Comparison of three paradigms: PEFT, compression then fine-tuning, and our
    joint adaptation and compression.}
    \label{fig:arch-compare}
    \vspace{-0.3in}
\end{figure}

As illustrated in \cref{fig:arch-compare}, unifying compression and adaptation is not as simple as chaining existing techniques. To achieve this synergy, we must define a search space that (i) preserves task-relevant capacity and (ii) supports effective task-specific updates under a fixed global parameter budget. Achieving both requires a task-aware mechanism that bridges the subspace favored by compression and the directions demanded by downstream objectives, so that the deployed low-rank model avoids the accuracy gaps caused by subspace misalignment. 

To this end, \textbf{JACTUS} (Joint Adaptation and Compression with a Task-aware Union of Subspaces) performs task-aware compression and tuning jointly within a single shared search space. Given a small calibration set, JACTUS estimates task-aware second-order statistics~\cite{amari1998natural, martens2015kfac} on both sides of each linear layer: the input covariance $\mathbf{C}_{x}$ and the pre-activation gradient covariance $\mathbf{C}_{g}$. Together, they induce a bilateral, downstream-informed metric that emphasizes loss-sensitive directions. We then construct an orthogonal union between the task-aware subspaces (from $\mathbf{C}_{x}, \mathbf{C}_{g}$) and the high-energy weight subspaces (from the SVD of $\mathbf{W}$), yielding left/right union bases ($\mathbf{Q}_{L}$/$\mathbf{Q}_{R}$) that cover both signals while removing redundancy. 

Within this union, we compute a projected low-rank approximation~\cite{golub2013matrix, benisrael2003generalized} by performing SVD on the small core $\mathbf{Q}_{L}^{\top}\mathbf{W}\mathbf{Q}_{R}$, whose spectrum directly guides rank allocation. Finally, under a global parameter budget, JACTUS greedily allocates ranks across layers based on marginal gain per parameter, and fine-tunes only the low-dimensional core matrix inside the fixed union subspace, as shown in \cref{fig:framework}. We further show that once the union subspaces are fixed, core-only optimization is \textbf{basis-invariant}: it explores the same attainable solutions as updating the factor bases, but at substantially lower cost.

This design closes the loop between compression and adaptation. Compared to SVD-only
pipelines~\cite{hsu2022languagemodelcompressionweighted,yuan2023asvd,wang2025svdllm,qinsi2025dobisvd},
JACTUS selects directions under a task-aware bilateral metric rather than weight
energy alone. Compared to PEFT approaches~\cite{hu2022lora,liu2024dora,wang2023orthogonalsubspacelearninglanguage,zhang2023adaptive,meng2024pissa,lingam2024svft,tian2024hydralora,zhang2023adaloraadaptivebudgetallocation,kopiczko2024vera,liu2024afloraadaptivefreezinglow},
it removes the need to retain full frozen weights at inference: the trained model
is the compact low-rank form. The global allocator further reconciles
heterogeneous layer sizes and spectra by spending budget where the next singular
direction provides the largest reduction in projected error per parameter.

We evaluate JACTUS on image classification with ViT-Base/Large
\cite{dosovitskiy2021image} and Swin-Base/Large \cite{liu2021Swin} across eight
datasets, and on Llama2-7B \cite{touvron2023llama2} for commonsense question answering
and mathematical reasoning. Under 80\% retained-parameter budgets, JACTUS outperforms strong 100\% PEFT baselines on both modalities (\eg, ViT-Base average 89.2\% \vs DoRA~\cite{liu2024dora} 87.9\%, and Llama2-7B commonsense average 80.9\% \vs DoRA 79.7\%), and degrades smoothly at 60\% and 40\% budgets while remaining competitive. Ablations show that (i) the joint union subspace strictly dominates weight-only or task-only subspaces; (ii) moderate calibration (about 1K examples) suffices for stable statistics; and (iii) global rank allocation yields sizable gains over uniform per-layer allocation.

The contributions of this work include:
\begin{itemize}
    \item We introduce \textbf{JACTUS}, a framework that unifies compression and
        adaptation by operating in an orthogonal union of weight and task-aware subspaces
        derived from input and gradient covariances.

    \item We build a \textbf{task-aware union of subspaces} where input/gradient covariances identify downstream-sensitive left/right directions and are fused with weight singular subspaces via orthogonalization. On top of that, we perform \textbf{cost-aware global rank allocation} to distribute ranks across layers under a fixed parameter budget.

    \item We prove a subspace invariance result showing that, once the working union subspaces are fixed, \textbf{optimizing only the core matrix} is equivalent to updating factor bases within those subspaces, enabling efficient fine-tuning.
\end{itemize}

    \section{Related Work}

\subsection{Low-Rank Compression via SVD}
Activation-aware SVD (ASVD)~\cite{yuan2023asvd} leverages activation statistics for
training-free calibration, but it is not explicitly aligned with downstream loss/gradient
geometry and can degrade under high compression or cross-task transfer.
Fisher-weighted SVD (FWSVD)~\cite{hsu2022languagemodelcompressionweighted} incorporates empirical
Fisher for task importance, yet selection is still formulated in a weight-space metric
and does not model the bilateral structure induced by inputs and gradients.
SVD-LLM~\cite{wang2025svdllm} stabilizes truncation via whitening and closed-form updates,
but continues to optimize weight reconstruction and typically decouples compression from adaptation.
Dobi-SVD~\cite{qinsi2025dobisvd} makes truncation differentiable and reconstructs weights via incremental PCA; however, its activation-only reconstruction can be sensitive to the calibration set and still
requires a separate fine-tuning stage to recover performance.

Overall, SVD-style compression often determines the retained subspace from weight energy or partial task cues, while adaptation is performed later in a different search space, which is an implicit source of subspace mismatch.
In contrast, JACTUS explicitly constructs a task-aware union of subspaces shared by compression and tuning, and couples it with a cost-aware global rank allocation across layers to preserve task-relevant capacity under a fixed deployment budget.

\subsection{Parameter-Efficient Fine-Tuning (PEFT)}
LoRA~\cite{hu2022lora} and its variants (e.g., DoRA~\cite{liu2024dora}, PiSSA~\cite{meng2024pissa},
OLoRA~\cite{wang2023orthogonalsubspacelearninglanguage}, AdaLoRA~\cite{zhang2023adaptive},
HydraLoRA~\cite{tian2024hydralora}, VeRA~\cite{kopiczko2024vera}, AFLoRA~\cite{liu2024afloraadaptivefreezinglow})
reduce training cost via low-rank adapters, but the deployed model typically retains full weights.
SVFT~\cite{lingam2024svft} re-parameterizes updates via a sparse intermediate matrix, yet storing both
full $\mathbf{U}$ and $\mathbf{V}$ can increase deployment storage.
More broadly, PEFT is primarily designed to approximate full fine-tuning without changing the base parameterization, and thus does not directly address the setting where weights are already low-rank compressed and the question is how to adapt efficiently within this constrained structure.

JACTUS targets this compressed-deployment regime: once the task-aware subspace is fixed, it performs tuning by updating only compact core matrices, which is equivalent to fine-tuning the full low-rank structure while keeping a PEFT-level memory footprint. This connects PEFT-style training efficiency with low-rank deployment without expanding the compressed parameterization.
    \section{Methodology}

\begin{figure*}[!t]
    \centering
    \vspace{-0.1in}
    \includegraphics[width=\linewidth]{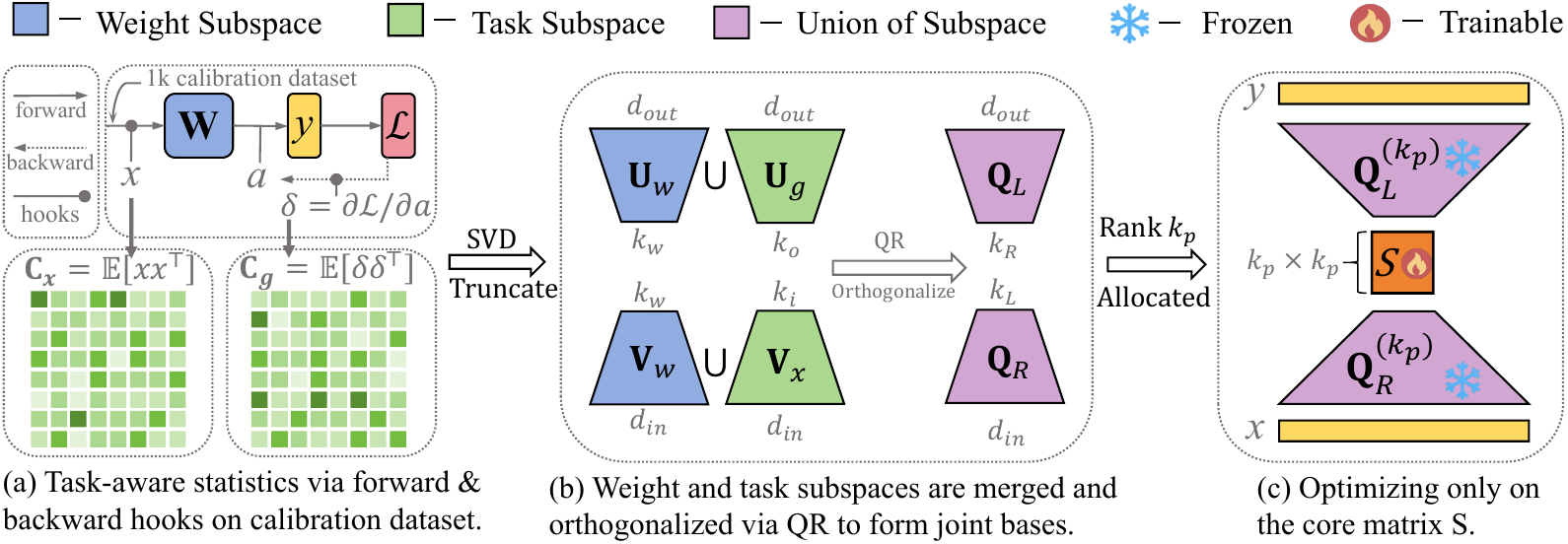}
    \vspace{-.2in}
    \caption{Overview of the proposed subspace-based adaptation. (a) Activation and
    gradient covariances $\mathbf{C}_{x}$ and $\mathbf{C}_{g}$ are estimated on the
    calibration set. (b) $\mathbf{U}_{w}, \mathbf{V}_{w}, \mathbf{U}_{g}, \mathbf{V}_{x}$ are
    obtained by applying SVD to $\mathbf{W}$, $\mathbf{C}_{x}$, and $\mathbf{C}_{g}$, then
    truncated by energy thresholds $\alpha,\beta,\gamma$ to ranks $k_{w}, k_{i},
    k_{o}$; the merged subspaces form joint bases
    $\mathbf{Q}_{L}, \mathbf{Q}_{R}$ with ranks $k_{L}, k_{R}$ satisfying
    $k_{w} + k_{i} = k_{L}$ and $k_{w} + k_{o} = k_{R}$. (c) A global rank allocator
    assigns rank $k_{p}$ to each linear layer.}
    \label{fig:framework}
    \vspace{-.2in}
\end{figure*}

To achieve an integrated
approach to model compression and task adaptation, the core process of JACTUS is
divided into three steps. As illustrated in \cref{fig:framework}, (i) we construct a task-aware union subspace guided by K-FAC statistics, merging pretrained weight structure with task sensitivity and performing projected low-rank approximation within this subspace; (ii) given a total parameter budget, we perform cost-aware global rank allocation via a greedy marginal-gain strategy to assign per-layer ranks; and (iii) we optimize core matrices only by leveraging low-rank reparameterization invariance, which enables efficient post-adaptation fine-tuning. We will elaborate
on our method focusing on an arbitrary linear layer $y = \mathbf{W}x$ (where the weight $\mathbf{W}\in \mathbb{R}^{d_{\text{out}}\times d_{\text{in}}}$) in a large foundation model.

\subsection{Task-aware Union of Subspaces Guided by K-FAC}

Traditional low-rank compression methods, such as SVD, primarily select
subspaces based on the singular spectrum of the weight matrix (\ie, its
cumulative energy). This selection criterion may have a weak correlation with
downstream task performance, leading to a subspace misalignment problem. To
address this issue, we introduce task-aware statistics using K-FAC(Kronecker-factorized Approximate curvature)~\cite{martens2015kfac}, extracting key information directly from the data distribution and loss function of the
downstream task using a small calibration set~\cite{Nagel2020AdaRound, Frantar2022GPTQ}.

\para{Task-aware Statistics Estimation.} We estimate two task-aware second-order
moments for a linear layer $a=\mathbf{W}x$: the input covariance
$\mathbf{C}_{x}=\mathbb{E}[x x^{\top}]\in\mathbb{R}^{d_{\text{in}}\times d_{\text{in}}}$
and the pre-activation gradient covariance
$\mathbf{C}_{g}=\mathbb{E}[\delta\delta^{\top}]\in\mathbb{R}^{d_{\text{out}}\times
d_{\text{out}}}$, where $\delta=\partial\mathbf{L}/\partial a$. From the natural-gradient
and K-FAC perspectives~\cite{amari1998natural, martens2015kfac}, the Gauss–Newton/Fisher
block for $\mathbf{W}$ in such a layer admits a Kronecker-factored approximation
\[
    F_{W} \;\approx\; \mathbb{E}[\delta\delta^{\top}]\;\otimes\;\mathbb{E}[x x^{\top}
    ] \;=\; \mathbf{C}_{g} \,\otimes\, \mathbf{C}_{x},
\]
so that $\mathbf{C}_{x}$ and $\mathbf{C}_{g}$ are precisely the two Kronecker factors
capturing input- and output-side curvature, respectively~\cite{amari1998natural,martens2015kfac}.
Selecting directions by the principal eigenspaces of $(\mathbf{C}_{x},\mathbf{C}_{g})$ therefore amounts to choosing directions that are important under a downstream-induced, bilateral (left/right) Fisher metric rather than by weight energy alone.

In practice, we accumulate $\mathbf{C}_{x}$ and $\mathbf{C}_{g}$ efficiently
from a small calibration set via streaming estimates with forward/backward hooks.
To enhance numerical stability and handle input directions potentially not covered
by the calibration set, we apply ridge regularization
\[
\small
    \mathbf{\tilde{C}}_{x} \;=\; \mathbf{C}_{x} + \tau_{x} \mathbf{I},\qquad \mathbf{\tilde{C}}
    _{g} \;=\; \mathbf{C}_{g} + \tau_{g} \mathbf{I},
\]
where the $\mathbf{I}$ is the identity matrix, and the regularization
coefficients $\tau_{x}=\epsilon\,\mathrm{tr}(\mathbf{C}_{x})/d_{\text{in}}$ and
$\tau_{g}=\epsilon\,\mathrm{tr}(\mathbf{C}_{g})/d_{\text{out}}$ are to maintain
scale invariance \cite{Tikhonov1963,LedoitWolf2004}.
This step explicitly encodes task sensitivity into left and right metric
matrices, laying the foundation for constructing task-aligned subspaces. To
construct a low-rank representation that both effectively reconstructs the
original weights and remains sensitive to critical directions for the downstream
task, we propose building a task-aware subspace union and performing a projected
low-rank approximation within it.

\para{Weight-Task Subspace Fusion}
Pretrained weights and task statistics provide complementary signals. Weight-derived
subspaces capture compressible global structure, while task-aware subspaces
capture directions most sensitive to the downstream loss. Using only one source introduces bias by either favoring high-energy but less relevant directions or chasing calibration noise, so we first identify both subspaces and then merge them. This yields a
search space that is simultaneously reconstructive and task-aligned under the
same parameter budget. There are two types of subspaces:
\begin{enumerate}
    \item \textbf{Weight Subspace}: We perform SVD on the original weight matrix
        $\mathbf{W}$ and select the subspaces spanned by the top $k_{w}$ left and
        right singular vectors (\eg, those capturing an energy threshold $\alpha$)
        \cite{golub2013matrix}, denoted as $\mathrm{span}(\mathbf{U}_{w})$ and $\mathrm{span}
        (\mathbf{V}_{w})$, respectively.

    \item \textbf{Task Subspace}: We perform eigendecomposition on the regularized
        covariance matrices $\mathbf{\tilde{C}}_{g}$ and $\mathbf{\tilde{C}}_{x}$
        and select the subspaces spanned by the top $k_{o}$ and $k_{i}$ eigenvectors
        (those capturing cumulative eigenvalue thresholds $\beta$ and $\gamma$),
        denoted as $\mathrm{span}(\mathbf{U}_{g})$ and
        $\mathrm{span}(\mathbf{V}_{x})$, respectively.
\end{enumerate}

To avoid subspace redundancy and ensure numerical stability, we use QR
decomposition to orthogonalize \cite{golub2013matrix} the concatenated subspaces.
This guarantees that the union subspace forms a well-conditioned orthogonal basis
that fully spans the combined representational directions of pretrained weights and
task sensitivity. Without orthogonalization, overlapping directions can lead to redundant
representations and unstable projections, which degrade both the low-rank
approximation and subsequent global rank allocation. Thus, we merge these two
types of subspaces via orthogonalization to construct the left and right
orthogonal union subspace bases, $\mathbf{Q}_{L}$ and $\mathbf{Q}_{R}$:
\[
\small
    \mathbf{Q}_{L}= \mathrm{orth}\big([\mathbf{U}_{w}, \mathbf{U}_{g}]\big) \in \mathbb{R}
    ^{d_{\text{out}}\times k_L}
\]
\[
\small
    \mathbf{Q}_{R}= \mathrm{orth}\big([\mathbf{V}_{w}, \mathbf{V}_{x}]\big) \in \mathbb{R}
    ^{d_{\text{in}}\times k_R}
\]
where $\mathrm{orth}(\cdot)$ denotes the orthogonalization of a set of column
vectors (\eg, via QR decomposition), and $k_{L}\le k_{w}+ k_{o}$ and
$k_{R}\le k_{w}+ k_{i}$ are the dimensions of the union subspaces. This union subspace
considers both weight energy and task sensitivity, and its representational
capacity is no worse than that of either individual subspace. This fundamentally
alleviates the performance recovery bottleneck encountered in post-compression fine-tuning,
which often operates in a suboptimal subspace.

Having obtained the target low-rank search subspace, we next require an initial point within this subspace to start optimization effectively; therefore, we initialize by projecting $\mathbf{W}$ onto the constructed subspace:
\[
\small
    \mathbf{W}_{\text{proj}}= \mathbf{Q}_{L}^{\top}\mathbf{W}\mathbf{Q}_{R}\in \mathbb{R}
    ^{k_L \times k_R}
\]
Next, we perform SVD on $\mathbf{W}_{\text{proj}}$, which has a much smaller
dimension than the original weight matrix: $\mathbf{W}_{\text{proj}}= \bar{\mathbf{U}}
\mathbf{\Sigma}\bar{\mathbf{V}}^{\top}$. This step not only reduces the
computational cost of SVD but also leads to a more stable estimation of the
singular spectrum. The squared values of the resulting singular values,
$\{\hat{\sigma}_{i}^{2}\}$, form the marginal gain curve for this layer within the
union subspace. The gain from increasing the rank from $k$ to $k+1$, measured as
the reduction in squared Frobenius reconstruction error, is approximated by
$\hat{\sigma}_{k+1}^{2}$ \cite{eckart1936approximation}.

\subsection{Cost-Aware Global Rank Allocator}
\label{sec:rank_allocation}

When compressing a large model to a low-rank form, the total number of remaining
parameters is a hard constraint. Prior practice (\eg, SVD-LLM \cite{wang2025svdllm})
often allocates the parameter budget independently across linear layers,
implicitly assuming that every rank increment is equally cost-effective.
This overlooks two facts: (i) layers have different dimensions, hence different
parameter costs per rank, and (ii) their singular spectra differ, so the
benefit (in reconstruction or downstream loss) of preserving the next singular
direction is heterogeneous. 

To address these limitations, we propose a cost-aware global rank allocator that shifts the paradigm from isolated layer-wise budgeting to a unified resource pool. Our motivation is grounded in the principle of maximal return on investment: every parameter should be allocated to the layer where it yields the highest performance gain relative to its size footprint. To systematically compare these heterogeneous benefits across layers, we quantify the \emph{marginal gain per unit parameter cost}. Concretely, for layer $p$ with dimensions
$d_{p}^{\text{out}}\times d_{p}^{\text{in}}$, increasing the rank by one in a $\mathbf{U}
_{p} \mathbf{V}_{p}^{\top}$ factorization adds $\Delta\text{params}_{p} = d_{p}^{\text{out}}
+ d_{p}^{\text{in}}$ parameters. In our task-aware union subspace, the projected
matrix $\mathbf{W}_{\text{proj},p}$ yields a stable spectrum $\{\hat{\sigma}_{p,i}
\}$; the reduction in squared Frobenius error when moving from rank $k_{p}$ to
$k_{p}{+}1$ is approximated by
$\text{gain}_{p}(k_{p}\!\to\!k_{p}{+}1)\ \approx\ \hat{\sigma}_{p,k_p+1}^{2}.$
We therefore rank all candidate increments by
\vspace{-5pt}
\[
    \rho_{p,k_p+1}\;=\;\frac{\hat{\sigma}_{p,k_p+1}^{2}}{\;\Delta\text{params}_{p}\;}
    ,
\]
and spend the budget on the largest ratios first, which naturally leads to our global
allocator that greedily selects rank increments in decreasing $\rho$ until the budget
is exhausted~\cite{Khuller1999BudgetedMC}. 

\begingroup
\setlength{\textfloatsep}{6pt}
\setlength{\intextsep}{6pt}
\setlength{\floatsep}{6pt}
\begin{algorithm}[!t]
    \caption{Global Rank Allocation under Parameter Budget $B$}
    \label{alg:global_rank_allocation} \SetKwInOut{Input}{Input} \SetKwInOut{Output}{Output}
    \Input{Number of layers $P$; layer sizes $\{d_{p}^{\text{out}}, d_{p}^{\text{in}}\}_{p=1}^{P}$; squared singular values $\{\hat{\sigma}_{p,i}^{2}\}$ from $\mathbf{W}_{\text{proj},p}$; budget $B$; bounds $\{k_{\min,p},k_{\max,p}\}$.}
    \Output{Allocated ranks $\{k_{p}\}_{p=1}^{P}$.} \BlankLine \For{$p\gets 1$ \KwTo $P$}{ $\text{cost}_{p}\gets d_{p}^{\text{out}}+ d_{p}^{\text{in}}$\; $k_{p}\gets k_{\min,p}$\; $L_{p} \gets \min\!\big(k_{\max,p},\, |\{\hat{\sigma}_{p,i}^{2}\}|\big)$\; }
    $\text{used}\gets \sum_{p=1}^{P}k_{p}\cdot \text{cost}_{p}$\; \BlankLine
    \textbf{Initialize} max-heap $\mathbf{H}$\; \For{$p\gets 1$ \KwTo $P$}{ \If{$k_{p}< L_{p}$}{ push $\big(\hat{\sigma}_{p,k_p+1}^{2}/\text{cost}_{p},\, p,\, k_{p}{+}1\big)$ into $\mathbf{H}$\; } }
    \While{$\mathbf{H}$ not empty}{ $(s, p, i) \gets \text{pop\_max}(\mathbf{H})$\; \If{$\mathrm{used}+ \mathrm{cost}_{p}> B$}{\textbf{continue}\;} $k_{p}\gets i$\; $\text{used}\gets \text{used}+ \text{cost}_{p}$\; \If{$i < L_{p}$}{ push $\big(\hat{\sigma}_{p,i+1}^{2}/\text{cost}_{p},\, p,\, i{+}1\big)$ into $\mathbf{H}$\; } }
    \Return $\{k_{p}\}_{p=1}^{P}$\;
\end{algorithm}
\endgroup


\begin{figure*}[!t]
    \centering
    \includegraphics[width=1\linewidth]{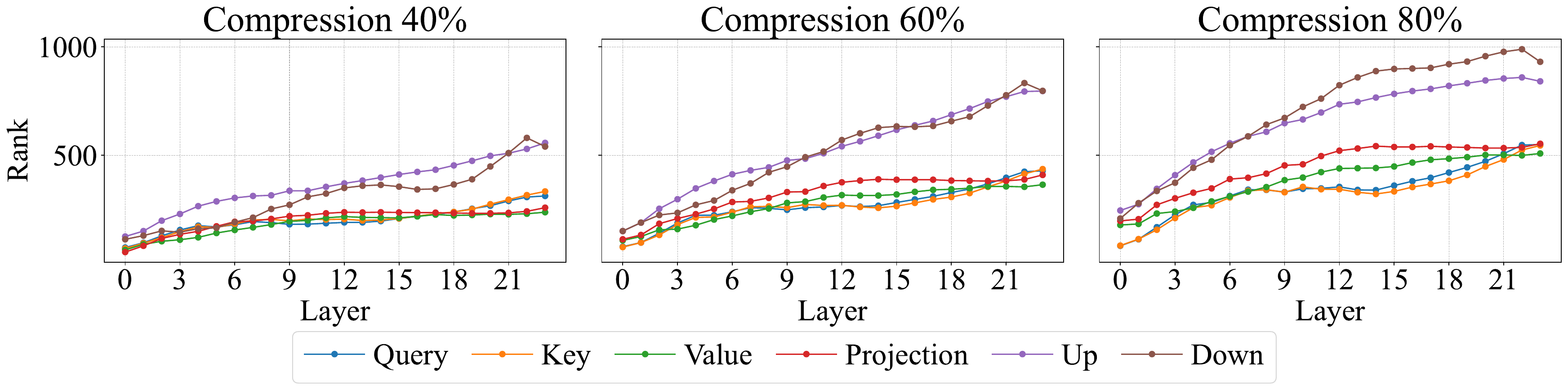}
    \vspace{-0.25in}
    \caption{Layer-wise ranks selected by the greedy global allocator for ViT-Large~\cite{dosovitskiy2021image} under 40\%, 60\%, and 80\% retained-parameter budget. We plot the ranks of attention (Query/Key/Value/Projection) and MLP (Up/Down) projections across layers. The calibration dataset is extracted from \textsc{CIFAR-100}~\cite{krizhevsky2009learning} training split.}
    \label{fig:ranks_vit}
    \vspace{-0.25in}
\end{figure*}

We detail this procedure in
\cref{alg:global_rank_allocation}. Given a target retained-parameter budget $c$, we induce a
global parameter budget B for the whole model. For each layer $p$, let
$k_{\min,p},k_{\max,p}$ be layer-wise bounds (defaults: $32$ and the number of
available singular values, respectively). We initialize $k_{p}\!\leftarrow\!k_{\min,p}$
and maintain a max-heap over the next available ratio $\rho_{p,k_p+1}$. At each
step, we pop the largest ratio that still fits the remaining budget, assign one
more rank to that layer, update the used parameters, and push its next candidate
(if any). This process continues until the heap is empty or the budget is exhausted.
The procedure is anytime and monotone: as $B$ grows, ranks only increase; it also
adapts to heterogeneous spectra and unequal layer sizes by normalizing gains with
per-rank costs.
The algorithm selects at most $\lfloor (B - \text{used}_{0})/\min_{p} \text{cost}
_{p} \rfloor$ rank increments, each with a heap operation of $O(\log P)$, yielding
an overall complexity $O\!\big((\sum_{p} k_{p})\log P\big)$. Because the
candidates within each layer follow a nonincreasing spectrum, the procedure naturally
spreads the budget over layers whose next singular directions are most valuable.

\cref{fig:ranks_vit} reports layer-wise ranks under the greedy allocator at a
40\%, 60\% and 80\% retained-parameter budget. Here, the Query, Key, Value, and Projection denote the self-attention weight matrices, while the Up and Down matrices are the first and second linear
maps of the MLP. The allocation is clearly non-uniform: Up and Down consistently
receive higher ranks than Query/Key/Value/Projection, and the assigned ranks generally
increase with depth across the network.

\subsection{Core-Only Optimization under Low-Rank Subspace Invariance}

After constructing task-aware left/right union subspaces
$\mathbf{Q}_{L}\subset\mathbb{R}^{d_{\mathrm{out}}}$ and
$\mathbf{Q}_{R}\subset\mathbb{R}^{d_{\mathrm{in}}}$ with
$k_{L}=\dim \mathbf{Q}_{L}$ and $k_{R}=\dim \mathbf{Q}_{R}$, the cost-aware global rank allocation
in the previous section selects a target rank $k_{p}\le \min\{k_{L},k_{R}\}$.
Henceforth, we work with the $k_{p}$-dimensional working subspaces $\mathbf{Q}_{L}
^{(k_p)}\subseteq \mathbf{Q}_{L}$,$\mathbf{Q}_{R}^{(k_p)}\subseteq \mathbf{Q}_{R}
,$ and all parameterizations use $k_{p}$ columns/rows drawn from these subspaces. While post-compression adaptation typically involves updating the low-rank factors $\mathbf{U}$ and $\mathbf{V}$ to recover performance, such unconstrained updates introduce redundant optimization trajectories and 
unnecessary computational overhead. Instead, we demonstrate that once the optimal 
task-aligned subspaces are identified, the parameter updates can be strictly 
confined to a compact core matrix $\mathbf{S}$. To rigorously justify this 
core-only optimization strategy, it is essential to abstract away from specific 
matrix parameterizations and examine the coordinate-free geometric space of 
achievable weight perturbations. The feasible set for downstream adaptation is 
defined as the following \emph{basis-free} tensor-product space $\mathcal{M}$:
\vspace{-5pt}
\[
    \small \mathcal{M}\!\left(\mathbf{Q}_{L}^{(k_p)},\mathbf{Q}_{R}^{(k_p)}\right) =\Bigl
    \{\, \mathbf{X}\in\mathbb{R}^{d_{\mathrm{out}}\times d_{\mathrm{in}}}\ \Big|\
    \begin{aligned}
         & \mathrm{Col}(\mathbf{X})\subseteq \mathbf{Q}_{L}^{(k_p)}, \\
         & \mathrm{Row}(\mathbf{X})\subseteq \mathbf{Q}_{R}^{(k_p)}
    \end{aligned}
    \Bigr\}.
\]
This space $\mathcal{M}$ encompasses all possible low-rank weight matrices whose 
column and row spans are strictly bounded within the selected task-aware subspaces, 
regardless of their explicit parameterization.
Equivalently, for any choice of bases
$\mathbf{U}\in\mathbb{R}^{d_{\mathrm{out}}\times k_p}$ and
$\mathbf{V}\in\mathbb{R}^{d_{\mathrm{in}}\times k_p}$ satisfying
$\mathrm{Span}(\mathbf{U})=\mathbf{Q}_{L}^{(k_p)}$ and
$\mathrm{Span}(\mathbf{V})=\mathbf{Q}_{R}^{(k_p)}$, we have the parameterization
\vspace{-5pt}
\[
\small
    \mathcal{S}(\mathbf{U},\mathbf{V}) =\Bigl\{\, \mathbf{USV}^{\top}\ \Big|\
    \begin{aligned}
         & \mathbf{U}\in\mathbb{R}^{d_{\mathrm{out}}\times k_p},\hspace{2mm}\mathrm{Span}(\mathbf{U})=\mathbf{Q}_{L}^{(k_p)}, \\
         & \mathbf{V}\in\mathbb{R}^{d_{\mathrm{in}}\times k_p},\quad \mathrm{Span}(\mathbf{V})=\mathbf{Q}_{R}^{(k_p)},        \\
         & \mathbf{S}\in\mathbb{R}^{k_p\times k_p}
    \end{aligned}
    \Bigr\}.
\]
Compression and adaptation therefore reduce to searching for an optimum \emph{inside
the fixed, task-specific space}
$\mathcal{M}\!\left(\mathbf{Q}_{L}^{(k_p)},\mathbf{Q}_{R}^{(k_p)}\right)$ rather
than roaming the full parameterization. The projected SVD step provides one convenient
orthonormal basis $(\mathbf{U},\mathbf{V})$ for this space together with an initial
core $\mathbf{S}$. Crucially, the space $\mathcal{M}\!\left(\mathbf{Q}_{L}^{(k_p)}
,\mathbf{Q}_{R}^{(k_p)}\right)$ is determined solely by $\bigl(\mathbf{Q}_{L}^{(k_p)}
,\mathbf{Q}_{R}^{(k_p)}\bigr)$—not by a particular choice of bases spanning them.
Switching bases only reparameterizes $\mathbf{S}$, leaving the attainable set of
solutions unchanged.

\para{Optimization objective.} Once the working subspaces are fixed, our
objective is to minimize the loss over
$\mathcal{M}\!\left(\mathbf{Q}_{L}^{(k_p)},\mathbf{Q}_{R}^{(k_p)}\right)$. Allowing
$\mathbf{U},\mathbf{V}$ to vary while keeping their spans fixed yields:
\begin{equation}
\vspace{-5pt}
\small
    \label{eq:subspace-opt}
    \begin{aligned}
        \min_{\mathbf{\widetilde{U}},\,\mathbf{\widetilde{S}},\,\mathbf{\widetilde{V}}}\; & \mathcal{L}\!\left(\mathbf{\widetilde{U}}\,\mathbf{\widetilde{S}}\,\mathbf{\widetilde{V}}^{\top}\right) \\
        \text{s.t.}\; \mathrm{Span}(\mathbf{\widetilde{U}})=                              & \mathbf{Q}_{L}^{(k_p)}, \mathrm{Span}(\mathbf{\widetilde{V}})=\mathbf{Q}_{R}^{(k_p)}.
    \end{aligned}
\end{equation}

\begin{theorem}
    \label{the:subspace}\textbf{Subspace invariance and parameterization
    equivalence.} Let $\mathbf{U}\in\mathbb{R}^{d_{\mathrm{out}}\times k_p}$ and
    $\mathbf{V}\in\mathbb{R}^{d_{\mathrm{in}}\times k_p}$ have full column rank,
    and define
    $\mathcal{S}(\mathbf{U},\mathbf{V})=\{\mathbf{USV}^{\top}: \mathbf{S}\in\mathbb{R}
    ^{k_p\times k_p}\}$. If $\mathbf{U'},\mathbf{V'}$ satisfy
    \[
    \small
        \begin{aligned}
            \mathrm{Span}(\mathbf{U'})                 & =\mathrm{Span}(\mathbf{U})=\mathbf{Q}_{L}^{(k_p)}, \\
            \text{and}\quad \mathrm{Span}(\mathbf{V'}) & =\mathrm{Span}(\mathbf{V})=\mathbf{Q}_{R}^{(k_p)},
        \end{aligned}
    \]

    \noindent then there exist invertible $\mathbf{R}_{L}\in\mathbb{R}^{k_p\times k_p}$
    and $\mathbf{R}_{R}\in\mathbb{R}^{k_p\times k_p}$ with
    $\mathbf{U'}=\mathbf{U}\mathbf{R}_{L}$ and
    $\mathbf{V'}=\mathbf{V}\mathbf{R}_{R}$, and
    \[
    \small
        \mathcal{S}(\mathbf{U'},\mathbf{V'}) \;=\; \mathcal{S}(\mathbf{U},\mathbf{V}
        ) \;=\; \mathcal{M}\!\left(\mathbf{Q}_{L}^{(k_p)},\mathbf{Q}_{R}^{(k_p)}\right
        ).
    \]
\end{theorem}


\para{Implications.} \cref{the:subspace} implies that any update trajectory that
keeps $\mathrm{Span}(\mathbf{U_t})=\mathbf{Q}_{L}^{(k_p)}$ and $\mathrm{Span}(\mathbf{V}
_{t})=\mathbf{Q}_{R}^{(k_p)}$ merely changes coordinates inside $\mathcal{M}\!\left
(\mathbf{Q}_{L}^{(k_p)},\mathbf{Q}_{R}^{(k_p)}\right)$; consequently we can fix
any bases $\mathbf{U},\mathbf{V}$ spanning
$\mathbf{Q}_{L}^{(k_p)},\mathbf{Q}_{R}^{(k_p)}$ and optimize only the core:
\vspace{-1mm}
\begin{equation}
\small
    \label{eq:core-only}\min_{\mathbf{S}\in\mathbb{R}^{k_p\times k_p}}\ \mathcal{L}
    \!\left(\mathbf{U}\,\textbf{S}\,\mathbf{V}^{\top}\right).
\end{equation}
\vspace{-1mm}
Our preceding steps construct task-aligned $\mathbf{Q}_{L},\mathbf{Q}_{R}$,
determine $k_{p}$ via global rank allocation, and induce the working subspaces $\mathbf{Q}
_{L}^{(k_p)},\mathbf{Q}_{R}^{(k_p)}$ that capture directions of high downstream sensitivity.
Training only $\mathbf{S}$ explores the same solution set as jointly updating
$(\mathbf{U},\mathbf{S},\mathbf{V})$ under the constraints above, preserving the
task-specific geometry while yielding the efficiency of updating a single $k_{p}{\times}
k_{p}$ matrix without sacrificing expressivity inside the intended space.




    \section{Experiment}
\label{sec:exp}
\subsection{Setup and Baselines}
We evaluate JACTUS on both vision and language tasks.
For vision, we fine-tune ViT-Base/Large~\cite{dosovitskiy2021image} and Swin-Base/Large~\cite{liu2021Swin}
(pre-trained on ImageNet-21k~\cite{deng2009imagenet}) on eight image-classification benchmarks.
For language, we evaluate Llama2-7B~\cite{touvron2023llama2} on eight commonsense QA tasks~\cite{talmor-etal-2019-commonsenseqa}
and mathematical reasoning (MATH, GSM8K)~\cite{yu2023metamath}.
We report results under a retained-parameter budget $c$, i.e., the percentage of backbone parameters kept in the deployed model.

\para{Tasks and metrics.}
For vision, we report top-1 accuracy on \textsc{CIFAR-100}~\cite{krizhevsky2009learning}, \textsc{EuroSAT}~\cite{helber2019eurosat},
\textsc{RESISC45}~\cite{cheng2017remote}, \textsc{StanfordCars}~\cite{krause20133d},
\textsc{FGVC Aircraft}~\cite{maji2013fine}, \textsc{DTD}~\cite{cimpoi2014describing},
\textsc{CIFAR-10}~\cite{krizhevsky2009learning}, and \textsc{OxfordPets}~\cite{parkhi2012cats},
plus macro averages (\cref{tab:vit,tab:swin}).
In \cref{tab:vit,tab:swin}, we abbreviate \textsc{FGVC Aircraft} as \textsc{Aircraft},
\textsc{StanfordCars} as \textsc{Cars}, and \textsc{OxfordPets} as \textsc{Pets}.
For language, we evaluate BoolQ~\cite{clark2019boolq}, PIQA~\cite{bisk2020piqa}, SIQA~\cite{sap2019socialiqa}, HellaSwag~\cite{zellers2019hellaswag}, WinoGrande~\cite{sakaguchi2020winogrande}, ARC-easy~\cite{clark2018arc}, ARC-challenge~\cite{clark2018arc}, and OBQA~\cite{mihaylov2018openbookqa}
with accuracy and their average (\cref{tab:commonsenseqa}), and report exact-match on MATH~\cite{hendrycksmath2021} and GSM8K~\cite{cobbe2021gsm8k}.
(\cref{tab:metamath}).

\para{Baselines.}
We compare against strong PEFT methods at 100\% budget (LoRA~\cite{hu2022lora}, DoRA~\cite{liu2024dora},
SVFT~\cite{lingam2024svft}, PiSSA~\cite{meng2024pissa} for vision), and two-stage
compress-then-adapt pipelines that apply low-rank compression followed by LoRA under the same budget:
SVD-LLM~\cite{wang2025svdllm}+LoRA and Dobi-SVD~\cite{qinsi2025dobisvd}+LoRA.
On vision, we implement an SVD-LLM-style compress-then-adapt baseline for ViT/Swin; we omit a Dobi-SVD variant on vision due to non-trivial adaptation overhead.
JACTUS estimates task-aware statistics using a 1024-sample calibration set from the training split.
Unless otherwise stated, all methods are run with matched data splits and training recipes at each budget
(details in the supplement).

\para{Environment.}
All experiments use 4$\times$ NVIDIA L40 GPUs with PyTorch~\cite{paszke2019pytorch},
Transformers~\cite{wolf2020transformers}, and DeepSpeed~\cite{Rasley2020DeepSpeedSO}.
Pretrained backbones are from the official Hugging Face hub (model IDs in the supplement).
We use full-precision AdamW~\cite{Loshchilov2017DecoupledWD}.
For \textbf{JACTUS}, we set $\alpha=0.95$ and $\beta=\gamma=0.99$.
For PEFT baselines, we follow their default settings (LoRA/DoRA/PiSSA rank 8; SVFT uses $d=8$ as in~\cite{lingam2024svft}).
Other training details (batch size, learning-rate schedules, etc.) are provided in the supplementary material.

\begin{table*}
    [!t]
    \centering
    \setlength{\tabcolsep}{2.0pt}
    \caption{Fine-tuning on ViT-Base and ViT-Large~\cite{dosovitskiy2021image}. Top-1
    accuracy (\%).}
    \vspace{-.12in}
    \label{tab:vit}
    \resizebox{\textwidth}{!}{
    \begin{tabular}{l l l c c c c c c c c c}
        \toprule Backbone                                        & Budget & Method                     & CIFAR-100     & EuroSAT       & RESISC45      & Cars  & Aircraft & DTD           & CIFAR-10      & Pets    & Avg.          \\
        \midrule \multirow{10}{*}{ViT-Base}                      & 100\%      & LoRA~\cite{hu2022lora}     & 92.0          & 98.4          & 94.7          & 77.1          & 57.3          & 78.0          & 98.8          & 93.1          & 86.2          \\
                                                                 & 100\%      & DoRA~\cite{liu2024dora}    & \textbf{92.6} & \textbf{99.2} & 95.0          & 81.8          & 62.8          & 79.1          & \textbf{99.0} & 93.8          & 87.9          \\
                                                                 & 100\%      & PiSSA~\cite{meng2024pissa} & 92.2          & 99.2          & 95.5          & 79.1          & 62.6          & 81.7          & 98.6          & \textbf{95.9} & 88.1          \\
                                                                 & 100\%      & SVFT~\cite{lingam2024svft} & 91.4          & 98.5          & 94.4          & 77.5          & 60.2          & 81.8          & 98.7          & 92.5          & 86.9          \\
        \cmidrule(lr){2-12}                                      
                                              & 80\%       & SVD-LLM+LoRA               & 91.6          & 98.3          & 94.5          & 76.8          & 57.0          & 77.7          & 98.6          & 92.7          & 85.9          \\
        & 80\%       & \textbf{JACTUS}                     & 92.3          & 99.1          & \textbf{96.0} & \textbf{84.7} & \textbf{68.2} & \textbf{81.8} & 98.8          & \textbf{93.4} & \textbf{89.2} \\
        \cmidrule(lr){2-12}                                      
                                              & 60\%       & SVD-LLM+LoRA               & 90.1          & 95.9          & 92.2          & 71.5          & 50.3          & 74.9          & 97.6          & 90.2          & 82.8          \\
        & 60\%       & \textbf{JACTUS}                     & 91.6          & 98.5          & 95.3          & 78.9          & 63.2          & 79.8          & 98.6          & 91.6          & 87.2          \\
        \cmidrule(lr){2-12}                                      
                                              & 40\%       & SVD-LLM+LoRA               & 81.0          & 92.4          & 84.4          & 54.5          & 50.8          & 62.9          & 94.9          & 83.4          & 75.5         \\
        & 40\%       & \textbf{JACTUS}                     & 85.6          & 95.4          & 91.9          & 67.0          & 55.1          & 70.5          & 96.8          & 88.2          & 81.3          \\
        \specialrule{1.2pt}{1pt}{1pt} \multirow{10}{*}{ViT-Large} & 100\%      & LoRA~\cite{hu2022lora}     & \textbf{94.9} & 99.0          & 94.7          & 83.3          & 64.5          & 81.5          & 99.1          & 94.8          & 89.0          \\
                                                                 & 100\%      & DoRA~\cite{liu2024dora}    & 94.3          & 98.9          & 95.4          & 84.2          & 69.8          & 81.7          & 99.1          & 94.8          & 89.8          \\
                                                                 & 100\%      & PiSSA~\cite{meng2024pissa} & 93.6          & 98.6          & 95.7          & 86.7          & 70.6          & \textbf{82.8}          & 99.0          & \textbf{95.7} & 90.3          \\
                                                                 & 100\%      & SVFT~\cite{lingam2024svft} & 93.7          & 98.8          & 95.4          & 84.9          & 68.7          & 81.3          & \textbf{99.3} & 93.3          & 89.4          \\
        \cmidrule(lr){2-12}                                      
                                              & 80\%       & SVD-LLM+LoRA               & 93.6         & 98.7          & 94.2          & 82.6          & 64.0          & 80.8          & 99.0          & 94.2          & 88.4          \\
        & 80\%       & \textbf{JACTUS}                     & 94.6          & \textbf{99.3} & \textbf{96.4} & \textbf{87.8} & \textbf{71.5} & 82.6 & \textbf{99.3} & 95.2          & \textbf{90.8} \\
        \cmidrule(lr){2-12}                                      
                                              & 60\%       & SVD-LLM+LoRA               & 92.8          & 97.5          & 93.6          & 81.3          & 61.9          & 78.6          & 98.1          & 92.2          & 87.0          \\
        & 60\%       & \textbf{JACTUS}                     & 93.5          & 98.7          & 95.7          & 86.9          & 68.8          & 81.4          & 98.8          & 93.2          & 89.6          \\
        \cmidrule(lr){2-12}                                      
                                              & 40\%       & SVD-LLM+LoRA               & 86.3          & 93.0          & 91.7          & 72.5          & 56.5          & 71.7          & 95.0          & 85.8          & 81.6          \\
        & 40\%       & \textbf{JACTUS}                     & 87.8          & 96.0          & 93.2          & 77.7          & 60.3          & 73.9          & 96.9          & 89.4          & 84.4          \\
        \bottomrule
    \end{tabular}
    }
    \vspace{-.12in}
\end{table*}

\begin{table*}
    [!t]
    \centering
    \setlength{\tabcolsep}{2.0pt}
    \caption{Fine-tuning on Swin-Base and Swin-Large~\cite{liu2021Swin}. Top-1 accuracy
    (\%).}
    \vspace{-.12in}
    \label{tab:swin}
    \resizebox{\textwidth}{!}{
    \begin{tabular}{l l l c c c c c c c c c}
        \toprule Backbone                                         & Budget & Method                     & CIFAR-100     & EuroSAT       & RESISC45      & Cars  & Aircraft & DTD           & CIFAR-10      & Pets    & Avg.          \\
        \midrule \multirow{10}{*}{Swin-Base}                      & 100\%      & LoRA~\cite{hu2022lora}     & 92.1          & 98.7          & 94.4          & 81.6          & 68.9          & 77.8          & 98.8          & 93.8          & 88.3          \\
                                                                  & 100\%      & DoRA~\cite{liu2024dora}    & 92.8          & 99.2          & 94.2          & 82.0          & 69.7          & 78.5          & \textbf{98.9} & 93.1          & 88.6          \\
                                                                  & 100\%      & PiSSA~\cite{meng2024pissa} & 93.1          & \textbf{99.3} & 94.5          & 82.1          & 70.4          & 79.2          & \textbf{98.9} & \textbf{93.3} & 88.9          \\
                                                                  & 100\%      & SVFT~\cite{lingam2024svft} & 92.0          & 99.1          & 94.4          & 82.1          & 69.9          & 80.1          & 98.6          & 93.3          & 88.7          \\
        \cmidrule(lr){2-12}                                       
                                               & 80\%       & SVD-LLM+LoRA               & 91.7          & 97.9          & 93.9          & 80.4          & 68.3          & 77.4          & 98.3         & 92.5          & 87.6          \\
        & 80\%       & \textbf{JACTUS}                     & \textbf{93.3} & 99.1          & \textbf{95.9} & \textbf{82.2} & \textbf{72.2} & \textbf{82.2} & \textbf{98.9} & 93.1          & \textbf{89.6} \\
        \cmidrule(lr){2-12}                                       
                                               & 60\%       & SVD-LLM+LoRA               & 90.3          & 96.9          & 92.5          & 79.7          & 68.0          & 76.5          & 97.1          & 90.8          & 86.5         \\
        & 60\%       & \textbf{JACTUS}                     & 92.2          & 98.7          & 95.0          & 81.8          & 71.9          & 81.6          & 98.2          & 92.4          & 89.0          \\
        \cmidrule(lr){2-12}                                       
                                               & 40\%       & SVD-LLM+LoRA               & 86.5          & 92.6          & 91.5          & 72.0          & 59.8          & 70.8          & 94.1          & 86.6          & 81.7          \\
        & 40\%       & \textbf{JACTUS}                     & 88.2          & 96.4          & 93.6          & 76.9          & 64.2          & 74.6          & 96.8          & 90.7          & 85.2          \\
        \specialrule{1.2pt}{1pt}{1pt} \multirow{10}{*}{Swin-Large} & 100\%      & LoRA~\cite{hu2022lora}     & 94.9          & 99.1          & 96.3          & 83.3          & 73.9          & 83.8          & 98.9          & 94.3          & 90.6          \\
                                                                  & 100\%      & DoRA~\cite{liu2024dora}    & 94.8          & 99.1          & 96.6          & 84.6          & 73.6          & 84.1          & 99.0          & 94.3          & 90.8          \\
                                                                  & 100\%      & PiSSA~\cite{meng2024pissa} & 94.6          & \textbf{99.3} & 96.5          & 84.0          & 74.1          & 84.2          & \textbf{99.1} & \textbf{94.5} & 90.8          \\
                                                                  & 100\%      & SVFT~\cite{lingam2024svft} & 94.2          & 99.1          & 95.8          & 83.2          & 73.8          & 84.1          & 98.9          & 94.3          & 90.4          \\
        \cmidrule(lr){2-12}                                       
                                               & 80\%       & SVD-LLM+LoRA               & 94.4          & 98.7          & 95.8          & 82.6          & 73.0          & 83.1          & 98.8          & 94.0          & 90.1          \\
        & 80\%       & \textbf{JACTUS}                     & \textbf{95.2} & 99.2          & \textbf{96.7} & \textbf{86.2} & \textbf{74.2} & \textbf{84.4} & \textbf{99.1} & 94.4          & \textbf{91.2} \\
        \cmidrule(lr){2-12}                                       
                                               & 60\%       & SVD-LLM+LoRA               & 93.4          & 98.2          & 94.9          & 81.3          & 71.6          & 81.9          & 98.0          & 92.3          & 89.0          \\
        & 60\%       & \textbf{JACTUS}                     & 94.0          & 98.5          & 95.3          & 84.7          & 72.9          & 83.1          & 98.7          & 93.3          & 90.1          \\
        \cmidrule(lr){2-12}                                       
                                               & 40\%       & SVD-LLM+LoRA               & 88.5          & 95.7          & 92.2          & 73.0          & 58.5          & 73.8          & 97.1          & 90.1         & 82.9          \\
        & 40\%       & \textbf{JACTUS}                     & 91.7          & 97.3          & 93.6          & 78.8          & 65.0          & 77.5          & 97.6          & 92.1          & 86.7          \\
        \bottomrule
    \end{tabular}
    }
    \vspace{-.25in}
\end{table*}

\subsection{Image Classification Results}
\label{sec:image}

\para{Main results.}
At an 80\% retained-parameter budget, JACTUS consistently outperforms the strongest 100\% PEFT baselines on all four backbones. Specifically, JACTUS reaches 89.2/90.8 average accuracy on ViT-Base/Large (\cref{tab:vit}) and 89.6/91.2 on Swin-Base/Large (\cref{tab:swin}),
exceeding the best 100\% PEFT counterparts (ViT-B: 88.1, ViT-L: 90.3; Swin-B: 88.9, Swin-L: 90.8).
As the budget tightens to 60\%, performance degrades gracefully and remains competitive with full-parameter PEFT (e.g., JACTUS stays above LoRA-100\% on ViT-Base/Large and Swin-Base, and is close on Swin-Large; see \cref{tab:vit,tab:swin}). At 40\%, the drop becomes more pronounced (e.g., 81.3 on ViT-Base), reflecting the expected difficulty under very tight retained-parameter budgets.

\para{Comparison to a two-stage compress-then-adapt baseline.}
Compared with SVD-LLM~\cite{wang2025svdllm}+LoRA under the same budgets, JACTUS is consistently better across backbones (\cref{tab:vit,tab:swin}), and the advantage increases as $c$ decreases.
For example, on ViT-Base the average gain grows from +3.3 at 80\% (89.2 vs 85.9) to +4.4 at 60\% (87.2 vs 82.8) and +5.8 at 40\% (81.3 vs 75.5). Similar trends of gap widdening hold for ViT-Large and Swin models (see \cref{tab:vit,tab:swin}). Overall, these results indicate that decoupling compression and adaptation amplifies subspace mismatch under tighter budgets, whereas JACTUS mitigates it by jointly aligning retained directions with task-sensitive signals.

\subsection{Scaling Up to Large Language Models}
\label{sec:llm} \textbf{Commonsense QA.}\quad At 80\% retained-parameter budget,
JACTUS achieves an averaged 80.9\% across eight tasks (\cref{tab:commonsenseqa}),
surpassing 100\% PEFT baselines (DoRA~\cite{liu2024dora} 79.7\%, SVFT~\cite{lingam2024svft}
78.5\%, LoRA~\cite{hu2022lora} 77.6\%). The improvement is broad-based, with a
notable margin on HellaSwag (JACTUS 93.0\% \vs DoRA 89.1\%). At 60\%
and 40\% retained-parameter budget, JACTUS remains robust (75.5\%, 58.9\%) and clearly
outperforms two-stage SVD-LLM+LoRA~\cite{wang2025svdllm} (68.0\%, 51.7\%) and
Dobi-SVD+LoRA~\cite{qinsi2025dobisvd} (71.5\%, 52.3\%).

\noindent
\begin{minipage}[t]{0.56\textwidth}
\vspace{0pt}
\textbf{Mathematical reasoning.} 
On MetaMath~\cite{yu2023metamath} fine-tuning, JACTUS
at 80\% reaches 13.8\% on MATH and 62.5\% on GSM8K (\cref{tab:metamath}),
outperforming 100\% PiSSA~\cite{meng2024pissa} (7.5\% / 53.2\%) and 100\% LoRA~\cite{hu2022lora}
(5.9\% / 43.2\%). As the budget tightens to 60\% and 40\%, performance
scales smoothly (11.1\% / 51.5\% and 8.8\% / 46.2\%), remaining competitive with
100\% PEFT overall.
\end{minipage}
\hspace{0.02\linewidth}
\begin{minipage}[t]{0.42\textwidth} 
\centering
\vspace{-10pt}
    \captionof{table}{MetaMath results on Llama2-7B (Accuracy \%, pass@1). LoRA and PiSSA results are from~\cite{meng2024pissa}.}
    \label{tab:metamath}
    \vspace{-5pt} 
    \vtop{\hbox{
    \resizebox{0.8\linewidth}{!}{
    \begin{tabular}[t]{l l c c}
        \toprule 
        Budget & Method & MATH & GSM8K \\
        \midrule 
        100\% & LoRA~\cite{hu2022lora} & 5.9 & 43.2 \\
        100\% & PiSSA~\cite{meng2024pissa} & 7.5 & 53.2 \\
        \midrule 
        80\% & \textbf{JACTUS} & \textbf{13.8} & \textbf{62.5} \\
        60\% & \textbf{JACTUS} & 11.1 & 51.5 \\
        40\% & \textbf{JACTUS} & 8.8 & 46.2 \\
        \bottomrule
    \end{tabular}
    }
    }}
\end{minipage}

\begin{table*}
    [!t]
    \centering
    \setlength{\tabcolsep}{4.5pt}
    \caption{CommonsenseQA on Llama2-7B~\cite{touvron2023llama2}. Accuracy (\%).
    Results of LoRA and DoRA are taken from DoRA~\cite{liu2024dora}.}
    \vspace{-.12in}
    \label{tab:commonsenseqa} \resizebox{\textwidth}{!}{
    \begin{tabular}{l l c c c c c c c c c}
        \toprule Budget & Method                                & BoolQ         & PIQA          & SIQA          & HellaSwag     & WinoGrande    & ARC-e         & ARC-c         & OBQA          & Avg.          \\
        \midrule 100\%      & LoRA~\cite{hu2022lora}                & 69.8          & 79.9          & 79.5          & 83.6          & 82.6          & 79.8          & 64.7          & 81.0          & 77.6          \\
        100\%               & DoRA~\cite{liu2024dora}               & \textbf{71.8} & \textbf{83.7} & 76.0          & 89.1          & 82.6          & 83.7          & 68.2          & \textbf{82.4} & 79.7          \\
        100\%               & SVFT~\cite{lingam2024svft}            & 71.2          & 81.0          & 76.3          & 89.6          & 81.1          & 81.8          & 66.5          & 81.4          & 78.5          \\
        \midrule 
        80\%                & SVD-LLM~\cite{wang2025svdllm}+LoRA    & 68.9          & 80.2          & 76.9          & 81.6          & 82.1          & 79.6          & 63.2          & 79.1          & 76.5          \\
        80\%                & Dobi-SVD~\cite{qinsi2025dobisvd}+LoRA & 69.2          & 81.7          & 77.6          & 82.1          & 82.5          & 78.3          & 64.3          & 79.7          & 76.9          \\
        80\%       & \textbf{JACTUS}                       & 70.7          & 82.9          & \textbf{80.6} & \textbf{93.0} & \textbf{83.3} & \textbf{85.4} & \textbf{71.7} & 79.6          & \textbf{80.9} \\
        \midrule 
        60\%                & SVD-LLM~\cite{wang2025svdllm}+LoRA    & 65.2          & 69.4          & 67.4          & 68.3          & 74.5          & 75.6          & 53.2          & 70.4          & 68.0          \\
        60\%                & Dobi-SVD~\cite{qinsi2025dobisvd}+LoRA & 66.0          & 70.7          & 72.0          & 75.3          & 78.0          & 78.7          & 58.9          & 72.6          & 71.5          \\
        60\%       & \textbf{JACTUS}                       & \textbf{66.8} & \textbf{79.8} & \textbf{77.4} & \textbf{87.8} & \textbf{78.5} & \textbf{76.9} & \textbf{61.5} & \textbf{75.4} & \textbf{75.5} \\
        \midrule 
        40\%                & SVD-LLM~\cite{wang2025svdllm}+LoRA    & 62.2          & 64.3          & 56.8          & 53.1          & 59.1          & 44.2          & 32.9          & 40.7          & 51.7          \\
        40\%                & Dobi-SVD~\cite{qinsi2025dobisvd}+LoRA & 64.4          & 61.8          & 64.3          & 50.5          & 59.3          & 47.0          & 30.8          & 40.0          & 52.3          \\
        40\%       & \textbf{JACTUS}                       & \textbf{62.6} & \textbf{64.4} & \textbf{66.6} & \textbf{61.1} & \textbf{60.0} & \textbf{56.9} & \textbf{43.0} & \textbf{56.6} & \textbf{58.9} \\
        \bottomrule
    \end{tabular}
    }
    \vspace{-.15in}
\end{table*}

\begin{table*}[t]
    \centering
    \setlength{\tabcolsep}{3.0pt}
    \renewcommand{\arraystretch}{1.05}

    \begin{minipage}[t]{0.95\textwidth}
        \centering

        \captionof{table}{
        Ablation of subspace choice at 60\% parameters on  CommonsenseQA.
        }
        \label{tab:ablate-subspace}
        \vspace{-0.08in}
        \resizebox{\linewidth}{!}{
        \begin{tabular}{l c c c c c c c c c}
            \toprule Method          & BoolQ         & PIQA          & SIQA          & HSwag.        & WG.           & ARC-e         & ARC-c         & OBQA          & Avg.          \\
            \midrule 
            Weight-Only              & 64.8          & 70.7          & 71.1          & 83.4          & 72.7          & 68.0          & 55.1          & 70.2          & 69.5          \\
            Task-Only                & 62.1          & 55.0          & 43.3          & 45.6          & 57.8          & 53.9          & 34.6          & 46.2          & 49.8          \\
            \textbf{JACTUS} & \textbf{66.8} & \textbf{79.8} & \textbf{77.4} & \textbf{87.8} & \textbf{78.5} & \textbf{76.9} & \textbf{61.5} & \textbf{75.4} & \textbf{75.5} \\
            \bottomrule
        \end{tabular}}
        \vspace{0.10in}

        \captionof{table}{Ablation of calibration set size at 60\% parameters on
        CommonsenseQA.}
        \label{tab:ablate-calib}
        \vspace{-0.08in}
        \resizebox{\linewidth}{!}{
        \begin{tabular}{l c c c c c c c c c}
            \toprule Size of Calib. & BoolQ         & PIQA          & SIQA          & HSwag.        & WG.           & ARC-e         & ARC-c         & OBQA          & Avg.          \\
            \midrule 512            & 65.3          & 77.0          & 76.8          & 85.6          & 76.0          & 73.8          & 61.3          & 74.6          & 73.8          \\
            \textbf{1024}           & \textbf{66.8} & \textbf{79.8} & 77.4          & \textbf{87.8} & 78.5          & \textbf{76.9} & 61.5          & \textbf{75.4} & \textbf{75.5} \\
            2048                    & 66.5          & 78.9          & \textbf{78.2} & 86.8          & \textbf{79.3} & 75.4          & \textbf{62.2} & 74.1          & 75.2          \\
            \bottomrule
        \end{tabular}
        }
        \vspace{0.10in}

        \captionof{table}{Ablation of rank allocation at 60\% parameters on CommonsenseQA.}
        \label{tab:ablate-alloc} 
        \vspace{-0.08in}
        \resizebox{\linewidth}{!}{
            \begin{tabular}{l c c c c c c c c c}
                \toprule Setting            & BoolQ         & PIQA          & SIQA          & HSwag.        & WG.           & ARC-e         & ARC-c         & OBQA          & Avg.          \\
                \midrule \textbf{Layerwise} & \textbf{66.8} & \textbf{79.8} & \textbf{77.4} & \textbf{87.8} & \textbf{78.5} & \textbf{76.9} & \textbf{61.5} & \textbf{75.4} & \textbf{75.5} \\
                \textbf{Uniform}            & 65.5          & 72.0          & 73.1          & 84.4          & 73.9          & 72.9          & 55.2          & 73.0          & 71.3          \\
                \bottomrule
            \end{tabular}
            }
    \vspace{-.2in}
    \end{minipage}

    \vspace{-0.10in}
\end{table*}

\subsection{Ablation Studies}
\label{sec:ablation} We ablate the role of the joint subspace, calibration size,
and rank allocation. Unless specified, the results are reported on CommonsenseQA~\cite{talmor-etal-2019-commonsenseqa}.

\para{Weight \vs task \vs joint subspace.} We compare (i) \textbf{Weight-Only} (pretrained
weight subspace), (ii) \textbf{Task-Only} (task gradient subspace), and (iii)
\textbf{JACTUS} (their orthogonal union) at 60\% retained-parameter budget
(see \cref{tab:ablate-subspace}). Averaged over tasks, JACTUS achieves 75.5\%,
while Weight-Only yields 69.5\% and Task-Only 49.8\%. The joint subspace
strictly contains the two single-source subspaces as special cases, hence it can
retain stable pretrained directions while admitting task-specific adaptations;
this explains both the sizable gap over Task-Only and the consistent
improvements over Weight-Only on harder datasets (\eg, HellaSwag, ARC-c)~\cite{talmor-etal-2019-commonsenseqa}.

\para{Calibration set size.} Varying the calibration size among \{512, 1024,
2048\} at 60\% retained-parameter budget \cref{tab:ablate-calib} shows that moderate
calibration is sufficient: accuracy rises from 73.8\% (512) to a peak of 75.5\%
(1024), and remains 75.2\% at 2048. Beyond $\sim$1k examples the returns become diminishing
and fluctuate slightly, indicating that the estimated task statistics are stable
once a modest coverage of the input/gradient distribution is reached. This property
is desirable for practical deployments where collecting a small, representative
calibration set is much cheaper than full dataset fine-tuning.

\para{Rank allocation: layer-wise \vs uniform.} We control for the overall budget
and compare uniform rank allocation (fixed 60\% per matrix) against layer-wise
allocation guided by task-aware statistics in \cref{tab:ablate-alloc}. Layer-wise JACTUS
improves the average from 71.3\% (Uniform) to 75.5\%. This indicates that our
cost-aware global rank allocation can concentrate rank on matrices with higher
effective energy and larger marginal gains for downstream objectives, thereby
supporting the effectiveness of the proposed approach.
    \section{Conclusion}
\label{sec:conclusion}

We presented \textbf{JACTUS}, a unified framework for joint task-aware adaptation and model compression. JACTUS bridges PEFT and post-hoc compression by constructing an orthogonal union of subspaces that combines the pre-trained weight structure with task-sensitive directions estimated from a small calibration set via task-aware statistics $(\mathbf{C}_{x}, \mathbf{C}_{g})$.
Within this joint subspace, we perform a projected low-rank approximation, allocate global ranks under a fixed parameter budget using a cost-aware marginal-gain strategy, and fine-tune only the compact core matrices. Across architectures (ViT, Swin, Llama2-7B) and modalities (vision, commonsense, math reasoning), JACTUS consistently improves accuracy at the same retained-parameter budget, including in high-compression regimes. Overall, our results suggest that compression and adaptation are best treated as a single optimization problem within a task-aware representation space.

\para{Future work.}
We plan to incorporate $(\mathbf{C}_{x}, \mathbf{C}_{g})$ more directly into the objective, enabling online updates of the subspaces and rank allocation during training rather than only at initialization. We also aim to design a statistics-driven global allocator that uses statistical sensitivity as a proxy for marginal gain, moving toward end-to-end task-aware joint compression and adaptation.
    
    \bibliographystyle{splncs04}
    \bibliography{main}

    \clearpage

    \appendix
    \renewcommand{\thetable}{A\arabic{table}} 
    \renewcommand{\thefigure}{A\arabic{figure}} 
    \renewcommand{\theequation}{A\arabic{equation}} 
    \setcounter{table}{0}
    \setcounter{figure}{0}
    \setcounter{equation}{0}
\section*{Appendix A: Basis Invariance and Core-only Equivalence}

\para{Subspace invariance and parameterization equivalence.} 
Let $\mathbf{U}\in\mathbb{R}^{d_{\mathrm{out}}\times k_p}$ and $\mathbf{V}\in\mathbb{R}^{d_{\mathrm{in}}\times k_p}$ have full column rank, and define
$\mathcal{S}(\mathbf{U},\mathbf{V})=\{\mathbf{USV}^{\top}: \mathbf{S}\in\mathbb{R}^{k_p\times k_p}\}$. If $\mathbf{U'},\mathbf{V'}$ satisfy
\[
\begin{aligned}
\mathrm{Span}(\mathbf{U'})&=\mathrm{Span}(\mathbf{U})=\mathbf{Q}_L^{(k_p)},\\
\text{and}\quad
\mathrm{Span}(\mathbf{V'})&=\mathrm{Span}(\mathbf{V})=\mathbf{Q}_R^{(k_p)},
\end{aligned}
\]
then there exist invertible $\mathbf{R}_L\in\mathbb{R}^{k_p\times k_p}$ and $\mathbf{R}_R\in\mathbb{R}^{k_p\times k_p}$ with
$\mathbf{U'}=\mathbf{U} \mathbf{R}_L$ and $\mathbf{V'}=\mathbf{V} \mathbf{R}_R$, and
\[
\mathcal{S}(\mathbf{U'},\mathbf{V'}) \;=\; \mathcal{S}(\mathbf{U},\mathbf{V}) \;=\; \mathcal{M}\!\left(\mathbf{Q}_L^{(k_p)},\mathbf{Q}_R^{(k_p)}\right).
\] 

\begin{proof}

Since $\mathrm{Span}(\mathbf{U'})=\mathrm{Span}(\mathbf{U})$ and both have full column rank $k_p$, there exists an invertible $\mathbf{R}_L\in\mathbb{R}^{k_p\times k_p}$ such that $\mathbf{U'}=\mathbf{U}\mathbf{R}_L$. Indeed, each column of $\mathbf{U'}$ is a linear combination of the columns of $\mathbf{U}$, so $\mathbf{U'}=\mathbf{U}\mathbf{R}_L$ for some $\mathbf{R}_L$; if $\mathbf{R}_L$ were singular, then $\mathrm{rank}(\mathbf{U'})\le \mathrm{rank}(\mathbf{R}_L)<k_p$, contradicting full column rank. The same argument gives an invertible $\mathbf{R}_R$ with $\mathbf{V'}=\mathbf{V}\mathbf{R}_R$.

We first show $\mathcal{S}(\mathbf{U'},\mathbf{V'})=\mathcal{S}(\mathbf{U},\mathbf{V})$. For any $\mathbf{S'}$,
\[
\mathbf{U'}\mathbf{S'}\mathbf{V'}^\top
= \mathbf{U}\mathbf{R}_L \mathbf{S'} (\mathbf{V}\mathbf{R}_R)^\top
= \mathbf{U}\bigl(\mathbf{R}_L \mathbf{S'} \mathbf{R}_R^\top\bigr)\mathbf{V}^\top,
\]
so $\mathcal{S}(\mathbf{U'},\mathbf{V'})\subseteq\mathcal{S}(\mathbf{U},\mathbf{V})$. Conversely, for any $\mathbf{S}$,
\[
\mathbf{U}\mathbf{S}\mathbf{V}^\top
= \mathbf{U'}\bigl(\mathbf{R}_L^{-1}\mathbf{S}(\mathbf{R}_R^{-1})^\top\bigr)\mathbf{V'}^\top,
\]
so $\mathcal{S}(\mathbf{U},\mathbf{V})\subseteq\mathcal{S}(\mathbf{U'},\mathbf{V'})$, proving equality.

Next, since $\mathrm{Span}(\mathbf{U})=\mathrm{Span}(\mathbf{Q}_L^{(k_p)})$ and both have full column rank, there exists an invertible $\mathbf{A}_L$ with $\mathbf{U}=\mathbf{Q}_L^{(k_p)}\mathbf{A}_L$, and similarly an invertible $\mathbf{A}_R$ with $\mathbf{V}=\mathbf{Q}_R^{(k_p)}\mathbf{A}_R$. Then for any $\mathbf{S}$,
\[
\mathbf{U}\mathbf{S}\mathbf{V}^\top
=  \mathbf{Q}_L^{(k_p)}\bigl(\mathbf{A}_L\mathbf{S}\mathbf{A}_R^\top\bigr)\bigl(\mathbf{Q}_R^{(k_p)}\bigr)^\top,
\]
hence $\mathcal{S}(\mathbf{U},\mathbf{V})\subseteq\mathcal{M}(\mathbf{Q}_L^{(k_p)},\mathbf{Q}_R^{(k_p)})$. Conversely, any element of $\mathcal{M}(\mathbf{Q}_L^{(k_p)},\mathbf{Q}_R^{(k_p)})$ has the form
\[
\mathbf{Q}_L^{(k_p)}\mathbf{T}\bigl(\mathbf{Q}_R^{(k_p)}\bigr)^\top
= \mathbf{U}\bigl(\mathbf{A}_L^{-1}\mathbf{T}(\mathbf{A}_R^{-1})^\top\bigr)\mathbf{V}^\top,
\]
so lies in $\mathcal{S}(\mathbf{U},\mathbf{V})$. Thus $\mathcal{S}(\mathbf{U},\mathbf{V})=\mathcal{M}(\mathbf{Q}_L^{(k_p)},\mathbf{Q}_R^{(k_p)})$, and combining with the first part gives the claimed equalities.

\end{proof}

\begin{table}[t]
\centering
\setlength{\tabcolsep}{6pt}
\caption{Vision benchmarks used in experiments: number of classes and split sizes.}
\label{tab:dataset-vision}
\begin{tabular}{lccc}
\toprule
Dataset & \#Classes & \#Train & \#Test \\
\midrule
\textsc{CIFAR-100}~\cite{krizhevsky2009learning} & 100 & 50{,}000 & 10{,}000 \\
\textsc{EuroSAT}~\cite{helber2019eurosat}        & 10  & 10{,}000 & 5{,}000  \\
\textsc{RESISC45}~\cite{cheng2017remote}        & 45  & 25{,}200 & 6{,}300  \\
\textsc{StanfordCars}~\cite{krause20133d}       & 196 & 8{,}144  & 8{,}041  \\
\textsc{FGVC Aircraft}~\cite{maji2013fine}      & 100 & 6{,}667  & 3{,}333  \\
\textsc{DTD}~\cite{cimpoi2014describing}        & 47  & 3{,}760  & 1{,}880  \\
\textsc{CIFAR-10}~\cite{krizhevsky2009learning} & 10  & 50{,}000 & 10{,}000 \\
\textsc{OxfordPets}~\cite{parkhi2012cats}       & 37  & 3{,}680  & 3{,}669  \\
\bottomrule
\end{tabular}
\end{table}

\begin{table}[t]
\centering
\setlength{\tabcolsep}{6pt}
\caption{Language generation benchmarks used in our experiments: training sources and evaluation split sizes.}
\label{tab:dataset-reasoning}
\begin{tabular}{lccc}
\toprule
Dataset & \#Train & \#Test \\
\midrule
\multicolumn{3}{l}{\textbf{CommonsenseQA fine-tuning experiment}} \\
\textsc{CommonsenseQA}~\cite{talmor-etal-2019-commonsenseqa}
  & 9{,}741 & 1{,}140 \\
\textsc{BoolQ}~\cite{clark2019boolq}
  & -- & 3{,}270 \\
\textsc{PIQA}~\cite{bisk2020piqa}
  & -- & 1{,}838 \\
\textsc{SIQA}~\cite{sap2019socialiqa}
  & -- & 1{,}954 \\
\textsc{HellaSwag}~\cite{zellers2019hellaswag}
  & -- & 10{,}042 \\
\textsc{WinoGrande}~\cite{sakaguchi2020winogrande}
  & -- & 1{,}267 \\
\textsc{ARC-Easy}~\cite{clark2018arc}
  & -- & 2{,}376 \\
\textsc{ARC-Challenge}~\cite{clark2018arc}
  & -- & 1{,}172 \\
\textsc{OBQA}~\cite{mihaylov2018openbookqa}
  & -- & 500 \\
\midrule
\multicolumn{3}{l}{\textbf{MetaMath fine-tuning experiment}} \\
\textsc{MetaMath}~\cite{yu2023metamath}
  & 395{,}000 & -- \\
\textsc{GSM8K}~\cite{cobbe2021gsm8k}
  & -- & 1{,}319 \\
\textsc{MATH}~\cite{hendrycksmath2021}
  & -- & 5{,}000 \\
\bottomrule
\end{tabular}
\end{table}

\begin{table*}[t]
\centering
\setlength{\tabcolsep}{4.0pt}
\caption{Hyperparameters for ViT-Base/Large~\cite{dosovitskiy2021image} fine-tuning on a single NVIDIA L40 GPU. Batch size is the global batch size and is shared across methods and compression rates for each backbone–dataset pair.}
\label{tab:vit_hparams}
\resizebox{\linewidth}{!}{
\begin{tabular}{l l c c c c c}
\toprule
Backbone & Dataset & Batch size & Learning rate & Warmup ratio & Weight decay & Epochs \\
\midrule
\multirow{8}{*}{ViT-Base}
& \textsc{CIFAR-100}   & 32 & $2\times 10^{-4}$ & 0.06 & 0.05 & 15 \\
& \textsc{EuroSAT}     & 32 & $1\times 10^{-4}$ & 0.06 & 0.05 & 12 \\
& \textsc{RESISC45}    & 32 & $5\times 10^{-4}$ & 0.06 & 0.05 & 12 \\
& \textsc{StanfordCars}&  32 & $5\times 10^{-4}$ & 0.10 & 0.05 & 15 \\
& \textsc{FGVC Aircraft} & 32 & $5\times 10^{-4}$ & 0.10 & 0.05 & 15 \\
& \textsc{DTD}         &  32 & $5\times 10^{-4}$ & 0.10 & 0.05 & 12 \\
& \textsc{CIFAR-10}    & 32 & $2\times 10^{-4}$ & 0.05 & 0.05 & 10 \\
& \textsc{OxfordPets}  &  32 & $2\times 10^{-4}$ & 0.10 & 0.05 & 12 \\
\midrule
\multirow{8}{*}{ViT-Large}
& \textsc{CIFAR-100}   &  32 & $1\times 10^{-4}$ & 0.06 & 0.05 & 15 \\
& \textsc{EuroSAT}     &  32 & $7\times 10^{-5}$ & 0.06 & 0.05 & 12 \\
& \textsc{RESISC45}    &  32 & $2\times 10^{-4}$ & 0.06 & 0.05 & 12 \\
& \textsc{StanfordCars}&  32 & $2\times 10^{-4}$ & 0.10 & 0.05 & 15 \\
& \textsc{FGVC Aircraft} & 32 & $2\times 10^{-4}$ & 0.10 & 0.05 & 15 \\
& \textsc{DTD}         &  32 & $2\times 10^{-4}$ & 0.10 & 0.05 & 12 \\
& \textsc{CIFAR-10}    &  32 & $1\times 10^{-4}$ & 0.05 & 0.05 & 10 \\
& \textsc{OxfordPets}  &  32 & $1\times 10^{-4}$ & 0.10 & 0.05 & 12 \\
\bottomrule
\end{tabular}
}
\end{table*}

\begin{table*}[t]
\centering
\setlength{\tabcolsep}{4.0pt}
\caption{Hyperparameters for Swin-Base/Large~\cite{liu2021Swin} fine-tuning on a single NVIDIA L40 GPU. Batch size is the global batch size and is shared across methods and compression rates for each backbone–dataset pair.}
\label{tab:swin_hparams}
\resizebox{\linewidth}{!}{
\begin{tabular}{l l c c c c c}
\toprule
Backbone & Dataset & Batch size & Learning rate & Warmup ratio & Weight decay & Epochs \\
\midrule
\multirow{8}{*}{Swin-Base}
& \textsc{CIFAR-100}   & 32 & $1\times 10^{-4}$ & 0.06 & 0.05 & 15 \\
& \textsc{EuroSAT}     & 32 & $1\times 10^{-4}$ & 0.06 & 0.05 & 12 \\
& \textsc{RESISC45}    & 32 & $1\times 10^{-4}$ & 0.06 & 0.05 & 12 \\
& \textsc{StanfordCars}&  32 & $2\times 10^{-4}$ & 0.10 & 0.05 & 15 \\
& \textsc{FGVC Aircraft} & 32 & $2\times 10^{-4}$ & 0.10 & 0.05 & 15 \\
& \textsc{DTD}         &  32 & $2\times 10^{-4}$ & 0.10 & 0.05 & 12 \\
& \textsc{CIFAR-10}    & 32 & $1\times 10^{-4}$ & 0.05 & 0.05 & 10 \\
& \textsc{OxfordPets}  &  32 & $1\times 10^{-4}$ & 0.10 & 0.05 & 12 \\
\midrule
\multirow{8}{*}{Swin-Large}
& \textsc{CIFAR-100}   &  32 & $7\times 10^{-5}$ & 0.06 & 0.05 & 15 \\
& \textsc{EuroSAT}     &  32 & $5\times 10^{-5}$ & 0.06 & 0.05 & 12 \\
& \textsc{RESISC45}    &  32 & $5\times 10^{-5}$ & 0.06 & 0.05 & 12 \\
& \textsc{StanfordCars}&  32 & $1\times 10^{-4}$ & 0.10 & 0.05 & 15 \\
& \textsc{FGVC Aircraft} & 32 & $1\times 10^{-4}$ & 0.10 & 0.05 & 15 \\
& \textsc{DTD}         &  32 & $1\times 10^{-4}$ & 0.10 & 0.05 & 12 \\
& \textsc{CIFAR-10}    &  32 & $7\times 10^{-5}$ & 0.05 & 0.05 & 10 \\
& \textsc{OxfordPets}  &  32 & $7\times 10^{-5}$ & 0.10 & 0.05 & 12 \\
\bottomrule
\end{tabular}
}
\end{table*}

\begin{table}[!t]
\centering
\setlength{\tabcolsep}{3.5pt}
\caption{Hyperparameters for Llama2-7B~\cite{touvron2023llama2} fine-tuning on CommonsenseQA~\cite{talmor-etal-2019-commonsenseqa} under different compression rates. All runs use a global batch size of 128, warmup ratio 0.03, weight decay 0, maximum sequence length 1024, and \texttt{bf16} training.}
\label{tab:csqa_hparams}
\begin{tabular}{l l c c}
\toprule
Comp. Rate & Method & Learning rate & Epochs \\
\midrule
100\% & \textbf{SVFT~\cite{lingam2024svft}}            & $5\times 10^{-4}$ & 1 \\
80\%  & \textbf{JACTUS}         & $1\times 10^{-4}$ & 1 \\
80\%  & \textbf{SVD-LLM~\cite{wang2025svdllm}+LoRA~\cite{hu2022lora}}   & $2\times 10^{-4}$ & 1 \\
80\%  & \textbf{Dobi\text{-}SVD~\cite{qinsi2025dobisvd}+LoRA} & $2\times 10^{-4}$ & 1 \\
60\%  & \textbf{JACTUS}         & $2\times 10^{-4}$ & 2 \\
60\%  & \textbf{SVD-LLM+LoRA}   & $4\times 10^{-4}$ & 2 \\
60\%  & \textbf{Dobi\text{-}SVD+LoRA} & $4\times 10^{-4}$ & 2 \\
40\%  & \textbf{JACTUS}         & $2\times 10^{-4}$ & 2 \\
40\%  & \textbf{SVD-LLM+LoRA}   & $4\times 10^{-4}$ & 2 \\
40\%  & \textbf{Dobi\text{-}SVD+LoRA} & $4\times 10^{-4}$ & 2 \\
\bottomrule
\end{tabular}
\end{table}

\begin{table}[!t]
\centering
\setlength{\tabcolsep}{6pt}
\caption{Hyperparameters for JACTUS on MetaMATH~\cite{yu2023metamath} with Llama2-7B~\cite{touvron2023llama2}. All runs use a global batch size of 128, warmup ratio 0.03, weight decay 0, maximum sequence length 1024, \texttt{bf16} training, cosine learning rate schedule.}
\label{tab:metamath_hparams}
\begin{tabular}{l l c c}
\toprule
Comp. Rate & Method & Learning rate & Epochs \\
\midrule
80\% & \textbf{JACTUS} & $1\times 10^{-4}$ & 1 \\
60\% & \textbf{JACTUS} & $2\times 10^{-4}$ & 2 \\
40\% & \textbf{JACTUS} & $2\times 10^{-4}$ & 2 \\
\bottomrule
\end{tabular}
\end{table}

\section*{Appendix B: Experimental Details and Hyperparameters}

\subsection*{Vision benchmarks}

For the vision experiments, we follow a standard image classification protocol on eight benchmarks. The number of classes and the train/test split sizes for each dataset are summarized in \cref{tab:dataset-vision}. For the language generation experiments, we conduct two separate fine-tuning protocols.
In the first protocol, we fine-tune on the CommonsenseQA~\cite{talmor-etal-2019-commonsenseqa} training split and evaluate the resulting model on eight established commonsense reasoning benchmarks.
In the second protocol, we fine-tune on the MetaMath~\cite{yu2023metamath} training set and directly evaluate the model on GSM8K~\cite{cobbe2021gsm8k} and MATH~\cite{hendrycksmath2021} for mathematical reasoning.
The training sources and the test split sizes for all datasets are summarized in \cref{tab:dataset-reasoning}.

For ViT-Base/Large~\cite{dosovitskiy2021image} and Swin-Base/Large~\cite{liu2021Swin} backbones, we use dataset-specific hyperparameters, but for a given backbone–dataset pair we keep the global batch size and training schedule identical across all methods and compression rates. The detailed hyperparameters for ViT-Base/Large and Swin-Base/Large are listed in \cref{tab:vit_hparams,tab:swin_hparams}, respectively. In all ViT and Swin experiments, JACTUS is applied to all linear projections inside the attention and MLP blocks: the query, key, value, and output projections, as well as the MLP up- and down-projection matrices. All remaining modules, including embeddings, layer normalizations, and classification heads, remain dense.

\subsection*{Joint subspace estimation, statistics, and rank allocation}

For each dataset, we estimate the joint subspace and second-order statistics using a small calibration set. Concretely, we uniformly sample 1{,}024 examples from the training split of the corresponding dataset, shuffle them, and run the (frozen) model to collect the input and gradient features. The covariance matrices $\mathbf{C}_x$ and $\mathbf{C}_g$ are estimated over these calibration examples using a batch size of 8 and an exponential running average.

For ViT~\cite{dosovitskiy2021image} and Swin~\cite{liu2021Swin} backbones, the statistics are accumulated in \texttt{fp32} to avoid numerical issues. For Llama2-7B~\cite{touvron2023llama2}, we instead use \texttt{bf16} statistics due to memory constraints at scale. In all experiments, we use a fixed ridge regularization coefficient of $10^{-5}$ when inverting the covariance matrices.

The rank allocator uses a per-layer lower bound $k_{\min}$ on the allowed rank to avoid degenerate solutions. For all ViT and Swin experiments we set $k_{\min}=32$. Under the compression regimes considered in the main paper (compression rates between $40\%$ and $80\%$), no layer in the vision or LLM experiments actually hits this lower bound, so the ranks are effectively determined by the optimizer. The global compression rate $c$ is defined over all trainable parameters but excludes the embedding layers: we first convert $c$ into a total rank budget $B$, subtract the budget that would be assigned to the embeddings, and then allocate the remaining budget across the attention and MLP matrices only.

\subsection*{Language generation benchmarks}

For language generation, we evaluate on CommonsenseQA~\cite{talmor-etal-2019-commonsenseqa} and MetaMATH~\cite{yu2023metamath} using Llama2-7B~\cite{touvron2023llama2} as the base model. The train/validation/test splits of these benchmarks follow the standard dataset splits. On CommonsenseQA, the SVFT~\cite{lingam2024svft} baseline does not report fine-tuning results for Llama2-7B. We therefore reproduce SVFT in the banded setting recommended by the original paper, using a band width $d=8$. For SVD-LLM~\cite{wang2025svdllm}+LoRA~\cite{hu2022lora} and Dobi-SVD~\cite{qinsi2025dobisvd}+LoRA, we start from the official compressed checkpoints at compression rates $80\%$, $60\%$, and $40\%$, and attach a LoRA adapter to each low-rank matrix. All such models are then fine-tuned on the CommonsenseQA training set. For Dobi-SVD, we use the \emph{unremapping} checkpoints released by the authors to ensure a fair comparison.

All language experiments use a unified global batch size of 128 and the same optimization hyperparameters: weight decay is set to 0 and the warmup ratio to 0.03. We implement this global batch size using data-parallel training with appropriate choices of per-device batch size, gradient accumulation steps, and the number of GPUs. For both CommonsenseQA and MetaMATH, we train with a maximum sequence length of 1024 and use \texttt{bf16} precision. The learning rate is scheduled with a cosine decay scheduler, and we fix the random seed to 42 for all runs.

During fine-tuning, we adopt an instruction-following prompt format. Given a task-specific instruction string \texttt{\{instruction\}}, we construct the input as:
\begin{verbatim}
Below is an instruction that describes a 
task. Write a response that appropriately 
completes the request.

### Instruction:{instruction}

### Response:
\end{verbatim}

At inference time, we decode with top-p sampling using $p=1$ (\ie, effectively greedy decoding) and temperature 0, with a maximum length of 1024 tokens.

The method- and compression-specific learning rates and numbers of epochs on CommonsenseQA are summarized in \cref{tab:csqa_hparams}. For MetaMATH, we reuse the same global batch size, warmup ratio, weight decay, and optimization settings, and only vary the learning rate and the number of epochs as shown in \cref{tab:metamath_hparams}.
\section*{Appendix C: Robustness to Random Seeds}
\label{sec:multiseed}

To evaluate optimization stability, we repeat selected configurations with three different random seeds under the same training recipe (data split, batch size, optimizer, and schedule). We report mean $\pm$ standard deviation over three runs. Baselines are reproduced under the same setup for fairness. The last column reports the overall average across all evaluated tasks in the corresponding benchmark suite, while only representative tasks are shown here for brevity.

\vspace{0.1in}

\begin{table}[t]
\centering

\begin{minipage}[t]{0.49\linewidth}
\centering
\caption{Language tasks on Llama2-7B at 60\% retained parameters with 100\% PEFT baselines. (mean$\pm$std over 3 runs).}
\label{tab:multiseed-language}
\vspace{-0.08in}
\resizebox{\linewidth}{!}{
\begin{tabular}{l c c c}
\toprule
Method & BoolQ & PIQA & Overall Avg. \\
\midrule
LoRA (100\%) & 69.8 $\pm$ 0.3 & 79.9 $\pm$ 0.2 & 77.6 $\pm$ 0.3 \\
DoRA (100\%) & 71.8 $\pm$ 0.2 & 83.7 $\pm$ 0.2 & 79.7 $\pm$ 0.2 \\
\midrule
SVD-LLM+LoRA (60\%) & 65.2 $\pm$ 0.3 & 69.4 $\pm$ 0.2 &
68.0 $\pm$ 0.2 \\
JACTUS (60\%) & \textbf{66.8 $\pm$ 0.2} & \textbf{79.8 $\pm$ 0.3} & \textbf{75.5 $\pm$ 0.2} \\
\bottomrule
\end{tabular}}
\end{minipage}
\hfill
\begin{minipage}[t]{0.49\linewidth}
\centering
\caption{Vision tasks on ViT-Base at 60\% retained parameters with 100\% baselines. (mean$\pm$std over 3 runs).}
\label{tab:multiseed-vision}
\vspace{-0.08in}
\resizebox{\linewidth}{!}{
\begin{tabular}{l c c c}
\toprule
Method & CIFAR-100 & RESISC45 & Overall Avg. \\
\midrule
LoRA (100\%) & 92.0 $\pm$ 0.1 & 94.4 $\pm$ 0.3 & 86.2 $\pm$ 0.2 \\
DoRA (100\%) & 92.6 $\pm$ 0.2 & 94.2 $\pm$ 0.2 & 87.9 $\pm$ 0.2 \\
\midrule
SVD-LLM+LoRA (60\%) & 90.1 $\pm$ 0.2 & 92.5 $\pm$ 0.3 & 85.9 $\pm$ 0.2 \\
JACTUS (60\%) & \textbf{91.6 $\pm$ 0.2} & \textbf{95.0 $\pm$ 0.3} & \textbf{87.2 $\pm$ 0.3} \\
\bottomrule
\end{tabular}}
\end{minipage}

\vspace{-0.1in}
\end{table}

\vspace{-0.05in}

Across both language and vision tasks, JACTUS consistently exhibits small variance (typically $<0.3$) while maintaining higher mean accuracy. This indicates stable convergence behavior under strict retained-parameter budgets.
\section*{Appendix D: Training Efficiency and Memory Usage}
\label{sec:training-efficiency}

Beyond deployment efficiency, we further evaluate the training-time
efficiency of JACTUS. Since JACTUS introduces a one-time
\emph{calibration stage} before fine-tuning, a natural question is whether
compression-aware adaptation increases the overall training cost.

\para{Training protocol.}
The training procedure consists of two stages:

\begin{itemize}
\item \textbf{Calibration.}
We first run the frozen pretrained model on a small calibration set
(1,024 training samples) to collect activation and gradient statistics for
estimating the joint subspace. This stage involves only forward and backward
passes, without parameter updates, and is performed only once for each task.

\item \textbf{Fine-tuning.}
After compression, adaptation is performed entirely in the reduced
low-rank parameter space.
\end{itemize}

Importantly, calibration is a one-time overhead and does not scale with the
number of training epochs.

\para{Wall-clock training time.}
Table~\ref{tab:eff} summarizes the wall-clock training time.
Although JACTUS introduces a short calibration stage, the subsequent
fine-tuning phase becomes more efficient due to the reduced number of
trainable parameters and lower memory traffic. For ViT-Base on
\textsc{CIFAR-100}, calibration takes only 0.04h, while the
fine-tuning time is reduced from 1.18h with standard LoRA/DoRA to
1.05h with JACTUS, resulting in a lower total training time of
1.09h. A similar trend is observed for Llama2-7B on
CommonsenseQA: JACTUS requires only 0.25h for calibration,
reduces fine-tuning time from 9.20h to 8.81h, and yields a
slightly lower total training time of 9.06h. These results indicate
that the calibration overhead is small in practice and can be largely offset
by the improved optimization efficiency in the compressed space.

\para{Peak memory usage.}
Because JACTUS performs optimization directly in the compressed subspace, it
reduces both the number of trainable parameters and the associated optimizer
states, leading to lower peak GPU memory usage during training. As shown in
Table~\ref{tab:eff}, for ViT-Base, peak memory is reduced from
5.9G to 4.3G; for Llama2-7B, it decreases from 42.8G to 40.2G. Therefore, JACTUS not only maintains
comparable or slightly better end-to-end training time, but also consistently
improves memory efficiency across both vision and language backbones.

\begin{table}[!t]
\centering
\caption{Training efficiency and memory comparison between JACTUS (60\% parameter budget) and standard PEFT methods (100\% parameter budget). ViT-Base is trained on CIFAR-100, while Llama2-7B is trained on CommonsenseQA.}
\setlength{\tabcolsep}{3.4pt}
\renewcommand{\arraystretch}{.98}
\begin{tabular}{llcccc}
\toprule
\textbf{Backbone} & \textbf{Method} & \textbf{Calibration} & \textbf{Tune} & \textbf{Total} & \textbf{Memory} \\
\midrule
\multirow{2}{*}{ViT-Base}
& LoRA/DoRA  & - & 1.18h & 1.18h & 5.9G \\
& \textbf{JACTUS}  & 0.04h & 1.05h & 1.09h & \textbf{4.3G} \\
\midrule
\multirow{2}{*}{Llama2-7B}
& LoRA/DoRA  & - & 9.20h & 9.20h & 42.8G \\
& \textbf{JACTUS}  & 0.25h & 8.81h & 9.06h & \textbf{40.2G} \\
\bottomrule
\end{tabular}
\label{tab:eff}
\vspace{-10pt}
\end{table}
\section*{Appendix E: Trade-offs between Compression and Performance}

In this section, we provide additional quantitative evidence on the trade-off between performance and computational efficiency achieved by JACTUS under different compression rates. Specifically, we report accuracy--compression curves for Llama2-7B~\cite{touvron2023llama2} averaged over eight language generation benchmarks (as summarized in Tab.~3), together with throughput--compression curves measured for Llama2-7B in a deployment-oriented setting.

\subsection*{Accuracy vs. Compression}

\Cref{fig:acc_compression} shows the overall average end-task accuracy across eight language generation benchmarks as a function of the compression rate for Llama2-7B~\cite{touvron2023llama2}. For each parameter budget, we compare the dense baseline (100\% parameters) with JACTUS under multiple compression rates.

Overall, JACTUS preserves most of the dense-model accuracy at moderate compression rates (\eg, 80\% and 60\% of the original parameters), while degrading gracefully under more aggressive compression, such as 40\%. The resulting curves are generally smooth, suggesting that the proposed global rank allocator can distribute model capacity effectively across layers rather than causing abrupt performance collapse under tight parameter budgets. Across these language benchmarks, JACTUS consistently outperforms naive low-rank baselines at the same nominal compression rate, indicating that cross-layer allocation guided by empirical covariance statistics is important for maintaining downstream accuracy.

\begin{figure}[t]
    \centering
    \includegraphics[width=0.55\linewidth]{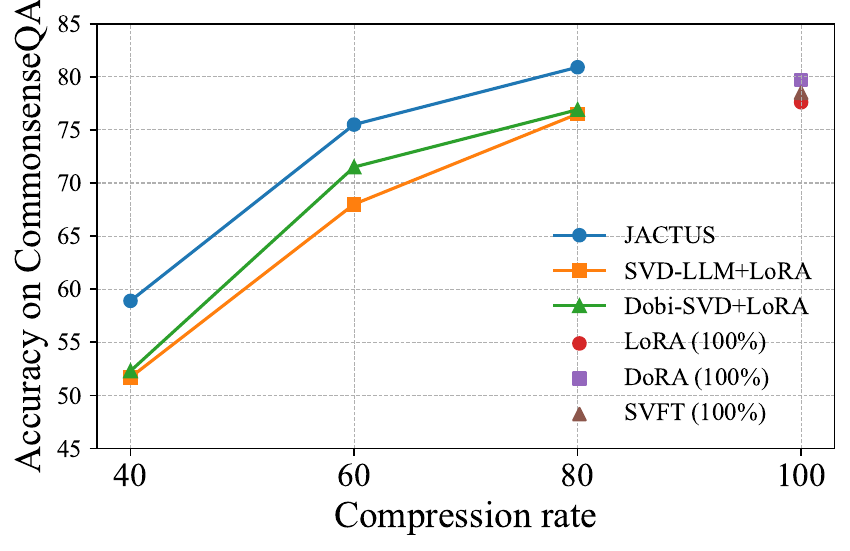}
    \caption{Overall average accuracy vs.\ compression rate on eight language generation benchmarks with Llama2-7B. JACTUS maintains competitive performance at moderate compression and degrades gracefully as compression becomes more aggressive.}
    \vspace{-0.1in}
    \label{fig:acc_compression}
\end{figure}

\subsection*{Throughput vs. Compression}
\begin{figure*}[!t]
    \centering
    \includegraphics[width=1\linewidth]{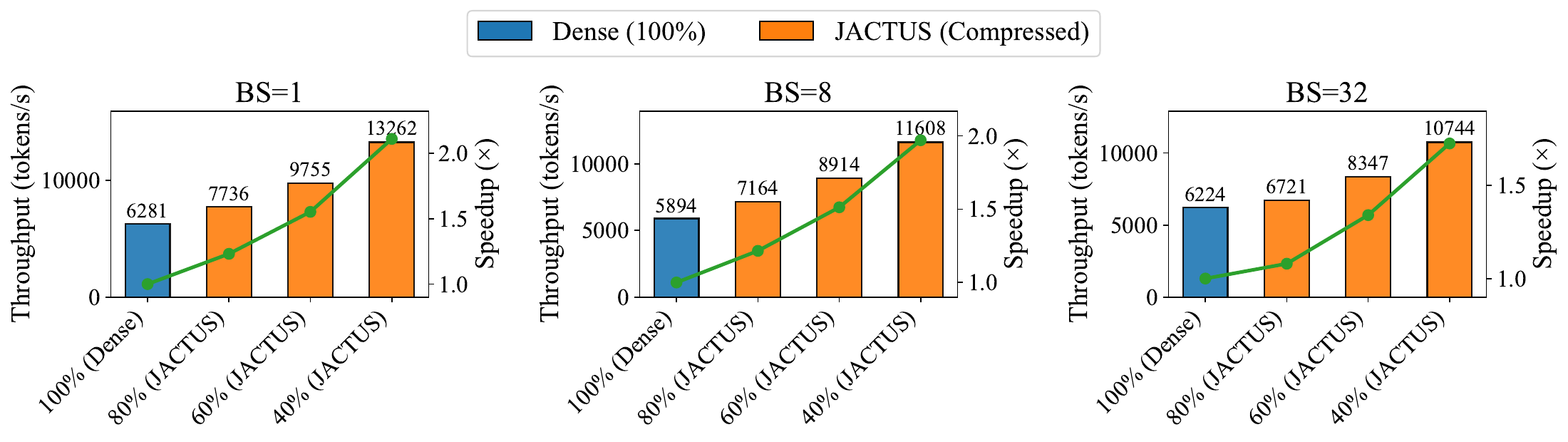}
    \caption{Throughput and speedup under different global batch sizes (BS) as a function of the compression rate on CommonsenseQA~\cite{talmor-etal-2019-commonsenseqa} with Llama2-7B~\cite{touvron2023llama2}. Blue bars show the throughput of the dense model and orange bars show the throughput of the model compressed by JACTUS (left axis, in tokens/s), while the green curve shows the speedup over the dense baseline (right axis).}
    \label{fig:throughput}
    \vspace{-0.1in}
\end{figure*}

To understand the impact of JACTUS on deployment-time efficiency, we also measure the decoding throughput of Llama2-7B~\cite{touvron2023llama2} on CommonsenseQA~\cite{talmor-etal-2019-commonsenseqa} under different compression rates. We first compress the model using a calibration set of 1{,}024 examples drawn from the CommonsenseQA training split, and then evaluate end-to-end generation throughput (in tokens per second) on the test set. \Cref{fig:throughput} reports the results for three global batch sizes, $\mathrm{BS}=1, 8, 32$.

Blue bars correspond to the dense (100\%) model, orange bars to JACTUS-compressed models at 80\%, 60\%, and 40\% parameter budgets, and the green curve shows the speedup relative to the dense baseline. Across all batch sizes, throughput increases monotonically as compression becomes more aggressive. At a moderate compression rate of 60\%, JACTUS already delivers approximately $1.2$--$1.5\times$ speedup over the dense model. When the budget is further reduced to 40\%, the speedup approaches $2\times$ for small and medium batch sizes and remains substantial (around $1.7\times$) even at $\mathrm{BS}=32$. Together with the accuracy--compression trends in \Cref{fig:acc_compression}, these results show that JACTUS can provide meaningful inference-time wall-clock gains while preserving competitive overall performance across language generation benchmarks.
\section*{Appendix F: Rank Allocation Visualizations}

In this section, we provide additional visualizations of the layer-wise ranks
selected by the JACTUS allocator. Since the main text already presents the rank allocation behavior on the vision backbone (ViT), here we provide the corresponding visualization for Llama2-7B. For this model, we run the greedy global rank allocation algorithm on a fixed calibration set and visualize the resulting ranks of
different submodules across layers under different compression rates. These
results complement the quantitative accuracy and throughput comparisons in
\cref{fig:acc_compression,fig:throughput} by illustrating how JACTUS
distributes its parameter budget within a large language model.

\subsection*{Llama2-7B on CommonsenseQA}

\begin{figure*}[!t]
    \centering
    \includegraphics[width=1\linewidth]{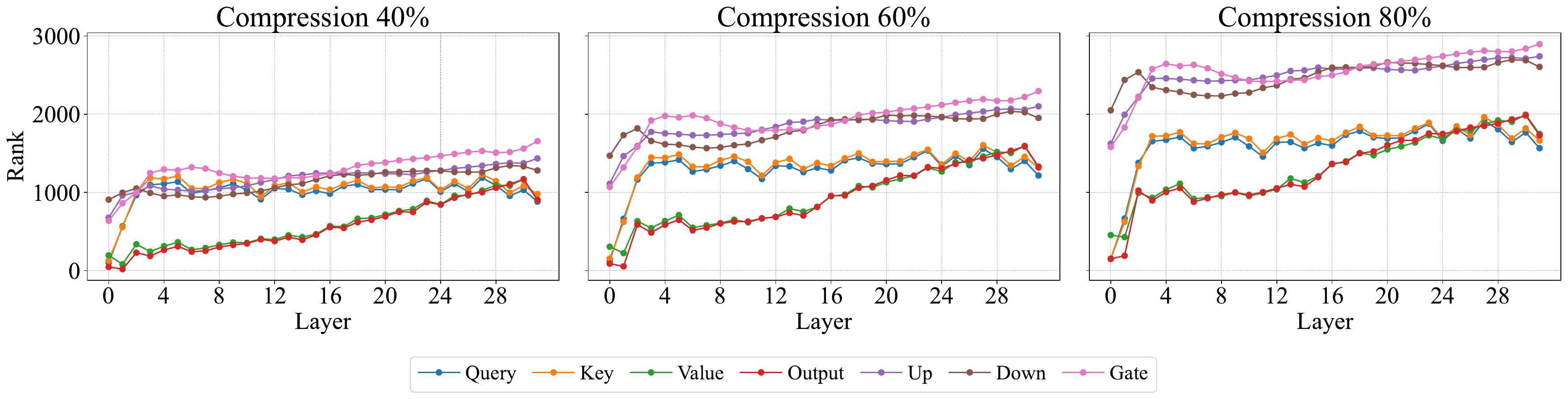}
    \caption{Layer-wise ranks selected by the greedy global allocator for Llama2-7B~\cite{touvron2023llama2} under 40\%, 60\%, and 80\% compression. We plot the ranks of attention (Query/Key/Value/Output) and MLP (Up/Down/Gate) projections as a function of layer index. The calibration set is built from the CommonsenseQA~\cite{talmor-etal-2019-commonsenseqa} training split.}
    \label{fig:ranks_llama}
    \vspace{-0.2in}
\end{figure*}

\cref{fig:ranks_llama} shows the layer-wise rank allocation for
Llama2-7B~\cite{touvron2023llama2} under 40\%, 60\%, and 80\% compression.
Specifically, we report the ranks assigned to the self-attention projections
(Query, Key, Value, and Output) and the MLP projections (Up, Down, and Gate)
at each layer. The calibration set is sampled from the CommonsenseQA
training split~\cite{talmor-etal-2019-commonsenseqa}.

Consistent with the trend observed for ViT in the main text, JACTUS tends to
assign larger ranks to deeper layers, while relatively suppressing the ranks
of the Key and Value projections. In addition, the MLP projections generally
receive higher ranks than several attention projections, suggesting that the
allocator identifies them as more parameter-sensitive components under the
same budget. As the compression rate becomes more aggressive, from 80\% to
40\%, the overall rank scale decreases accordingly, while the qualitative
allocation pattern remains largely stable. This indicates that JACTUS
consistently prioritizes later layers and capacity-intensive modules even
under tighter parameter budgets.

    %
    %

\end{document}